\documentclass{article}

\usepackage[final]{neurips_2022}
\usepackage{graphicx}
\usepackage[utf8]{inputenc} 
\usepackage[T1]{fontenc}    
\usepackage{hyperref}       
\usepackage{url}            
\usepackage{booktabs}       
\usepackage{amsfonts}       
\usepackage{nicefrac}       
\usepackage{microtype}      
\usepackage{xcolor}
\usepackage{amsmath}
\usepackage{physics}
\usepackage{algpseudocode}
\usepackage{algorithm}
\usepackage{float}
\usepackage{tabularx}
\usepackage{subfigure}
\usepackage{lipsum}
\usepackage{color,soul}
\usepackage{amsfonts}
\usepackage{algpseudocode}
\usepackage{graphicx}
\usepackage{epstopdf}\usepackage{pdfpages}
\usepackage{amsopn}
\usepackage{booktabs}
\usepackage{multirow}
\usepackage{xcolor}
\usepackage{lipsum}
\usepackage{soul}
\usepackage{tikz}

\usepackage{array, makecell}
\setcellgapes{5pt}
\usepackage{amssymb,amsmath,amsthm}
\usepackage[title]{appendix}
\usepackage{mathtools} 
\newtheorem{theorem}{Theorem}

\newtheorem{lemma}[theorem]{Lemma}
\usepackage{todonotes}

\newtheorem*{remark}{Remark}

\usepackage{mathtools} 

\title{Learning Interface Conditions in Domain Decomposition Solvers}

\author{%
  Ali Taghibakhshi \\
  Mechanical Science and Engineering\\
  University of Illinois Urbana-Champaign\\
  Urbana, IL 61801, USA \\
  \texttt{alit2@illinois.edu} \\
  \And
  Nicolas Nytko \\
  Computer Science\\
  University of Illinois Urbana-Champaign\\
  Urbana, IL 61801, USA \\
  \texttt{nnytko2@illinois.edu} \\
  \And
  Tareq Zaman \\
  Scientific Computing Program\\
  Memorial University of Newfoundland\\
  and Labrador\\
  St. John's, NL, Canada\\
  \texttt{tzaman@mun.ca}\\
  \And
  Scott MacLachlan \\
  Mathematics and Statistics\\
  Memorial University of Newfoundland\\
  and Labrador\\
  St. John's, NL, Canada\\
  \texttt{smaclachlan@mun.ca}\\
  \And
  Luke Olson \\
  Computer Science\\
  University of Illinois Urbana-Champaign\\
  Urbana, IL 61801, USA \\
  \texttt{lukeo@illinois.edu} \\
  \And
  Matthew West \\
  Mechanical Science and Engineering\\
  University of Illinois Urbana-Champaign\\
  Urbana, IL 61801, USA \\
  \texttt{mwest@illinois.edu} \\
}

\begin{document}

\maketitle

\begin{abstract}
Domain decomposition methods are widely used and effective in the approximation
of solutions to partial differential equations.  Yet the \textit{optimal} construction of
these methods requires tedious analysis and is often available only in
simplified, structured-grid settings, limiting their use for more complex problems.
In this work, we generalize optimized Schwarz domain decomposition methods to unstructured-grid problems,
using Graph Convolutional Neural Networks (GCNNs) and unsupervised learning
to learn optimal modifications at subdomain interfaces. A key
ingredient in our approach is an improved loss function, enabling
effective training on relatively small problems, but robust
performance on
arbitrarily large problems, with computational cost linear in problem size. The
performance of the learned linear solvers is compared with both classical and
optimized domain decomposition algorithms, for both structured- and unstructured-grid problems.
\end{abstract}

\section{Introduction}

Domain decomposition methods (DDMs)~\cite{MR2104179, MR1857663, MR3450068} are
highly effective in solving the linear and nonlinear systems of equations that
arise from the numerical approximation of solutions to partial differential equations (PDEs).  While most effective on elliptic boundary-value problems, DDMs can also be applied to nonlinear problems, either using their nonlinear variants, or successively solving linearizations. Time-dependent problems are normally solved by using a time stepping algorithm in the time domain for implicit methods, which require the solution of a spatial problem for each time step. Of these methods, Schwarz methods are particularly popular given their
relative simplicity and ease of parallelization.  The common theme is to break the
global problem into subproblems, derived either by projection or by
discretizing the same PDE over a physical subdomain, and to use solutions on the
subdomains as a preconditioner for the global discretization. Classical Schwarz
methods generally make use of Dirichlet or Neumann boundary conditions for
these subdomain problems, while Optimized Schwarz Methods (OSMs) aim to improve
the convergence of the algorithm by using more general interface conditions~\cite{GHN_optimized_2000}.
Notably,~\cite{MR2867648} demonstrates that optimal,
but non-local, interface conditions exist for more general decompositions.

Much of the OSM literature considers only one-level additive
Schwarz methods, although multilevel extensions do exist.
For one-level methods (i.e., domain decomposition approaches
without a ``coarse grid''), restricted additive Schwarz (RAS)
approaches~\cite{cai1999restricted} are arguably the most common; optimized restricted
additive Schwarz (ORAS) methods are considered in~\cite{st2007optimized}.
The OSM idea has also
been extended to asynchronous Schwarz methods~\cite{MR3679933}, where the
computations on each subdomain are done using the newest information available
in a parallel computing environment without synchronizing the solves on each
subdomain.

With a recent focus on machine
learning (ML) techniques for solving PDE systems~\cite{raissi2019physics,
li2020fourier}, there
is also effort to apply learning-based methods to improve the
performance of iterative solvers for PDEs, including DDM and algebraic
multigrid (AMG) methods.  Within AMG methods, ML techniques have been applied
to learning interpolation operators~\cite{greenfeld2019learning,luz2020learning}
and to coarse-grid selection in reduction-based
AMG~\cite{taghibakhshi2021optimization}. Of particular note here is the loss
function employed in~\cite{greenfeld2019learning,luz2020learning}, where they
use unsupervised learning to train a graph neural network to minimize the
Frobenius norm of the error-propagation operator of their iterative method.
Within DDM, significant effort has been invested in combining ML techniques
with DDM, as in~\cite{heinlein2021combining}, where two main families of approaches are given:
1) using ML within classical DDM  
methods to improve convergence, and
2) using deep neural networks,  
such as Physics Informed Neural Networks
(PINNs)~\cite{raissi2019physics}, as a discretization module and solver for DDM
problems. In~\cite{heinlein2019machine}, a fully connected neural network is
used to predict the geometric locations of constraints for coarse
space enhancements in an adaptive Finite Element Tearing and Interconnecting-Dual
Primal (FETI-DP) method. Using the continuous formulation of DDM, the so-called
D3M~\cite{li2019d3m} uses a variational deep learning solver, implementing
local neural networks on physical subdomains in a parallel fashion.
Likewise, Deep-DDM~\cite{li2020deep} utilizes PINNs to discretize and solve
DDM problems, with coarse space corrections~\cite{mercier2021coarse} being used
to improve scalability.

In this paper, we advance DDM-based solvers by developing a
framework for learning optimized Schwarz preconditioners.  A key
aspect of this is reconsidering the loss function to use a
more effective relaxation of the ideal objective function than Frobenius
norm minimization~\cite{greenfeld2019learning, luz2020learning}.
Moreover, the approach introduced here offers an opportunity to reconsider existing
limitations of optimized Schwarz methods, where optimal parameter choice
is based on Fourier analysis and requires a highly regular subdomain structure,
such as in the classical cases of square domains split into two equal subdomains
or into regular grids of square subdomains.  Our framework learns the optimized
Schwarz parameters via training on small problem sizes, in a way that
generalizes effectively to large problems, and in a way that allows us to
consider both irregular domains and unstructured grids, with no extraordinary
limitations on subdomain shape.  Furthermore, the evaluation time of our
algorithm scales linearly with problem size. This allows significant
freedom in defining optimized Schwarz methods, in comparison to classical approaches,
allowing us to explore the potential benefits of optimized Schwarz over
classical (restricted) additive methods on unstructured grids for the first
time.

\section{Background}\label{sec:background}

Let $\Omega \subset \mathbb{R}^2$ be an open set, and
consider the positive-definite Helmholtz problem
\begin{equation}\label{eq:Helmholtz}
  Lu = (\eta - \Delta)u = f \quad\text{in $\Omega$},
\end{equation}
with inhomogeneous Dirichlet conditions imposed on the boundary $\partial\Omega$. In~\eqref{eq:Helmholtz}, the parameter $\eta>0$ represents a \textit{shift} in the Helmholtz problem.
In the numerical results below, we consider both finite-difference discretizations of~\eqref{eq:Helmholtz} on regular grids, as well as piecewise linear finite-element (FE) discretizations on arbitrary triangulations.  In both cases, we denote the set of degrees of freedom as $D$, and note that these are in a one-to-one correspondence with the nodes of the underlying mesh. Consider a decomposition of $D$ into non-overlapping subdomains $D_i^{0}, i\in\{1, 2, \ldots, S\}$ such that each node is contained within exactly one subdomain $D_{i}^{0}$, yielding $\cup D_i^{0} = D$. In this subdomain notation, the superscript denotes the amount of overlap in the subdomains, which is zero for the non-overlapping subdomains that we first consider. Let $R_{i}^0$ be the restriction operator onto the set of degrees of freedom (DoFs) in $D_{i}^0$, and let $\left(R_{i}^0\right)^{T}$ be the corresponding extension operator from $D_{i}^{0}$ into set $D$. Then, an FE discretization of the Helmholtz problem leads to a linear system of the form $Ax = b$, where $A $ is the global stiffness matrix and $A_{i}^0 = R_{i}^0A\left(R_{i}^0\right)^{T}$ is the subdomain stiffness matrix for $D_{i}^0$.  We note that alternate definitions to the Galerkin projection for $A_i^0$ are possible, and are commonly considered in optimized Schwarz settings (as noted below).

In the case of restricted additive Schwarz (RAS)~\cite{cai1999restricted}, the subdomains are extended to allow for overlap: nodes near the ``boundary'' of their subdomain are potentially included in two or more subdomains. We denote the amount of overlap by $\delta \in \mathbb{N}$, defining subdomains $D_{i}^{\delta}$ recursively, by $D_{i}^{\delta} = D_{i}^{\delta-1} \cup \left\{ j \mid a_{kj}\neq 0 \text{ and }k\in D_{i}^{\delta-1}\right\}$ for $\delta > 0$.  The conventional RAS preconditioner is defined as
\begin{equation}
M_{\text{RAS}} = \sum\limits_{i=1}^{S} (\tilde{R}^{\delta}_{i})^{T}(A_{i}^{\delta})^{-1}{R_{i}^{\delta}},
\end{equation}
where $R_{i}^{\delta}$ is the standard restriction operator to subdomain $\Omega_{i}^{\delta}$, $\tilde{R}_{i}^{\delta}$ is a modified restriction operator from $D$ to the DoFs in $D_i^{\delta}$ that takes nonzero values only for DoFs in $D_i^0$, and $A_{i}^{\delta} = {(R_{i}^{\delta})}^{T}AR_{i}^{\delta}$. Figure~\ref{fig:imp_example} shows an example unstructured grid with two subdomains and overlap $\delta = 1$.

In an optimized Schwarz setting, we modify the subdomain systems, $A_i^\delta$.   Rather than using a Galerkin projection onto $D_i^\delta$, we rediscretize~\eqref{eq:Helmholtz} over the subdomain of $\Omega$ corresponding to these DoFs, imposing a Robin-type boundary condition on the boundary of the subdomain. We define this matrix to be $\tilde{A}_{i}^{\delta} = A_i +L_{i}$, where $A_i$ is the term resulting from discretization of~\eqref{eq:Helmholtz} with Neumann boundary conditions, and $L_i$ is additional the term resulting from the Robin-type condition, as in:
\begin{align}
                                \text{Dirichlet: } u & = g_{\text{D}}(x), &
             \text{Neumann: } \vec{n}\cdot\nabla u & = g_{\text{N}}(x), &
  \text{Robin: } \alpha u + \vec{n}\cdot\nabla u & = g_{\text{R}}(x),
\end{align}
where $\vec{n}$ is the outward normal to the edge on the boundary and $g$ denotes inhomogeneous data.

The matrix $L_{i}$ has the same dimensions as $A_{i}$, so that $\tilde{A}_{i}^{\delta}$ is well-defined. However, it has a significantly different sparsity pattern, with nonzero entries only in rows/columns corresponding to nodes on the boundary of subdomain $D_{i}^{\delta}$. In practice, we identify a cycle or path in the graph corresponding to $A_{i}$ with the property that every node in the cycle is on the boundary of $D_{i}^{\delta}$ but not the boundary of $D$ (the discretized domain), and then restrict the nonzeros in $L_{i}$ to the entries corresponding to the edges in this cycle/path (including self-edges, corresponding to entries on the diagonal of $L_{i}$).
%
\begin{figure}
  \centering
  \includegraphics{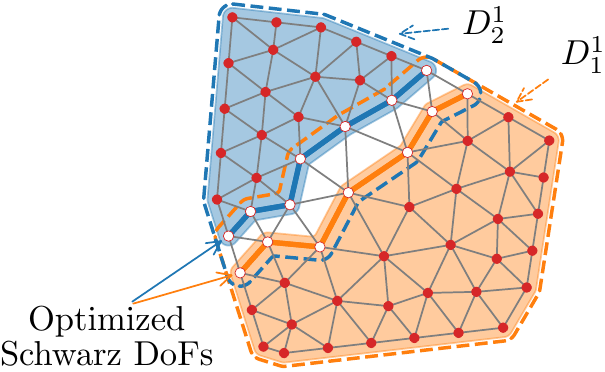}
    \caption{Two subdomains with overlap $\delta=1$ on a 58-node unstructured grid. The blue and orange shading denotes the original non-overlapping partitions ($D_1^0$ and $D_2^0$), while the blue and orange dashed outlines show the overlapping subdomains ($D_1^1$ and $D_2^1$).  Nodes belonging to only one subdomain are marked with a solid circle, while those in white belong to both subdomains.  The connections in the boundary matrices $L_1$ and $L_2$ are denoted by the edges shaded in blue (for $L_1$) and orange (for $L_2$).}\label{fig:imp_example}
\end{figure}

Using this notation, the ORAS preconditioner can be written as
\begin{equation}\label{eq:M_ORAS}
M_{\text{ORAS}} = \sum\limits_{i = 1}^S\left(\tilde{R}_{i}^\delta\right)^{T}\left(\tilde{A}_{i}^{\delta}\right)^{-1}R_{i}^{\delta}.
\end{equation}
Our aim is to learn the values in the matrices $L_i$, aiming to outperform the classical choices of these values predicted by Fourier analysis, in the cases where those values are known, and to learn suitable values for cases, such as finite-element discretizations on unstructured grids, where no known optimized Schwarz parameters exist.  We optimize the values for the case of stationary (Richardson) iteration, but evaluate the performance of the resulting methods both as stationary iterations and as preconditioners for FGMRES\@.

\paragraph{Graph Neural Networks (GNNs):} Applying learning techniques to graph structured data necessitates stepping beyond multilayer perceptron (MLP) and conventional convolutional neural networks (CNN) to a type of network that leverages the underlying graph nature of the problem, namely graph convolutional neural networks (GCNNs). GCNNs are typically divided into two categories: spectral and spatial~\cite{wu2020comprehensive}. Spectral GCNNs, first introduced by Bruna et al.~\cite{bruna2013spectral}, consider a graph Laplacian eigenbasis and define convolution as diagonal operators. As such, spectral GCNN methods suffer from time complexity problems due to the necessity for the eigendecomposition of the Laplacian matrix. Nevertheless, in follow-up work~\cite{defferrard2016convolutional, kipf2016semi}, remedies have been proposed to mitigate this. Unlike spectral methods, spatial GCNNs consider local propagation of information in graph domains as a convolution graph. One popular framework is the message passing neural network (MPNN)~\cite{gilmer2017neural}, which is based on sharing information among neighbor nodes in each round of a convolution pass. This has been generalized~\cite{battaglia2018relational} by introducing a configurable graph network block structure consisting of node and edge convolution modules and a global attribute. In an effort to alleviate computational complexity of GCNNs~\cite{du2017topology}, topology adaptive graph convolution networks (TAGCN) can be constructed by defining learnable filters.
This is not only computationally simpler, but also allows for adapting to the topology of the graphs when scanning them for convolution.

\section{Method}\label{sec: method}

\subsection{Optimization problem and the loss function}

Throughout this paper, we use $\|\cdot\|$ to denote the $\ell^2$ norm of a matrix or vector, $\|A\|_{F}$ for the Frobenius norm of $A$, and $\rho(A)$ as the spectral radius of $A$.  The optimization problem that we seek to solve is to find optimal values for the entries in the matrices $L_i$, constrained by given sparsity patterns, to minimize $\rho(T)$, where $T = I-M_{\text{ORAS}}A$ is the error-propagation operator for the stationary iteration corresponding to $M_{\text{ORAS}}$ defined in~\eqref{eq:M_ORAS}.  The spectral radius $\rho(T)$ corresponds to the \textit{asymptotic convergence factor} of the stationary iteration, giving a bound on the asymptotic convergence of the method.  Formally defined in terms of the extremal eigenvalue of $T^T T$ (since $T$ is not symmetric), direct minimization of $\rho(T)$ is difficult since backpropagation of an eigendecomposition is numerically unstable~\cite{wang2019backpropagation}.  To overcome this, Luz et al.~\cite{luz2020learning} propose to relax the minimization of $\rho(T)$ (for a similar AMG preconditioner) to minimizing the Frobenius norm of $T$, $\|T\|_F$.  In our experiments, however, we find that this is insufficient, leading to preconditioners that do not scale.  One reason is that while the Frobenius norm is an upper bound on the spectral radius, it does not appear to be a suitably ``tight'' bound for use in this context (see Section~\ref{sec:result_uns} and Figure~\ref{fig:better_loss}).  Instead, we use a relaxation inspired by Gelfand's formula, that $\rho(T) = \displaystyle\lim_{K\rightarrow\infty} \|T^K\|^{\frac{1}{K}}$, and the common bound that
\begin{equation}
  \label{eq:rho_bound}
\rho(T) \leq \|T^K\|^{\frac{1}{K}} = \sup_{x\neq 0} \left(\frac{\|T^K x\|}{\|x\|}\right)^\frac{1}{K} = \sup_{x: \|x\|=1} (\|T^{K}x\|)^\frac{1}{K}
\end{equation}
for some finite $K\in \mathbb{N}$.  This results in the optimization problem
\begin{equation}
  \label{eq:loss_func}
  \min_{\substack{L_i, i=1,2,\ldots, S \\ \text{sparsity of $L_i$}}} \sup_{x: \|x\|=1}  \|T^{K}x\|.
\end{equation}

\subsection{Numerical evaluation of the loss function}

We denote the action of evaluating the GNN by $f^{(\theta)}$ (where $\theta$ represents the network parameters), and consider a discretized problem with DoF set $D$, of size $n$. The set $D$ can be decomposed into subdomains either by using fixed geometric choices of the subdomain (e.g., for finite-difference discretizations), using the METIS graph partitioner~\cite{karypis1998fast}, or a $k$-means-based clustering algorithm (best known as Lloyd's algorithm which has $O(n)$ time complexity)~\cite{bell2008algebraic, lloyd1982least}.  For unstructured problems, we use a $k$-means-based algorithm, decomposing $D$ to subdomains $D_{i}$ for $i = 1, 2, \ldots, S$, with overlap $\delta$; see Supplementary Materials for details.  The GNN then takes $D$ and its decomposition as inputs, as well as sparsity constraints on the matrices $L_i$ for $i=1,2,\ldots, S$, and outputs values for these matrices:
\begin{equation}
\label{eq:GNNinout}
{L_{1}^{(\theta)}, L_{2}^{(\theta)}, \ldots, L_{S}^{(\theta)}} \leftarrow f^{(\theta)}(D).
\end{equation}
Using the learned subdomain interface matrices, we then obtain the modified MLORAS (Machine Learning Optimized Restricted Additive Schwarz) operator, $M^{(\theta)}_{\text{ORAS}}$, simply using $\tilde{A}_i^\delta = A_i + L_i^{(\theta)}$ in~\eqref{eq:M_ORAS}.  We denote the associated error propagation operator by $T^{(\theta)} = I - M^{(\theta)}_{\text{ORAS}}A$.

While Gelfand's formula and the associated upper bound in~\eqref{eq:rho_bound} are valid in any norm, it is natural to consider them with respect to the $\ell^2$ norm in this setting.  However, this raises the same issue as encountered in~\cite{luz2020learning}, that it generally requires an eigendecomposition to compute the norm.  To avoid this, we use a stochastic sampling of $\left\|\left(T^{(\theta)}\right)^K\right\|$, generated by the sample set $X\in\mathbb{R}^{n\times m}$ for some $m\in\mathbb{N}$, given as
\begin{equation}
X = [x_1, x_2, \ldots, x_m], \forall_{j} \; x_j\sim \mathbb{R}^{n} \; \text{uniformly}, \|x_j\| = 1.
\end{equation}
Here, we randomly select $m$ points uniformly on a unit sphere in $\mathbb{R}^{n}$, which can be done using the method in~\cite{box1958note}. We then define
\begin{equation}
Y^{(\theta)} = \left\{\left\|\left( T^{(\theta)}\right)^{K}x_1\right\|, \left\|\left(T^{(\theta)}\right)^{K}x_2\right\|, \ldots, \left\|\left(T^{(\theta)}\right)^{K}x_m\right\|\right\},
\end{equation}
taking each column of $X$ as the initial guess to the solution of the homogeneous problem $Ax=0$ and taking $K$ steps of the stationary algorithm to generate $\left( T^{(\theta)}\right)^{K}x_j$.  Since we normalize each column of $X$ to have $\|x_j\| = 1$, the value of $\left\|\left( T^{(\theta)}\right)^{K}x_j\right\|$ serves as a lower bound for $\left\|\left( T^{(\theta)}\right)^{K}\right\|$.  Thus,
taking the maximum of the values in $Y$ provides a practical loss function that we use below, defining
\begin{equation}\label{eq:numeric_loss}
\mathcal{L}^{(\theta)} = \max(Y^{(\theta)}).
\end{equation}
We note similarities between this loss function and that used in~\cite{KATRUTSA2020112524}, but that we are able to use the maximum of $Y^{(\theta)}$ (giving a better approximation to the norm) in our context instead of averaging.

Ultimately, the cost of our algorithm depends strongly on the chosen values of $m$ and $K$.  For sufficiently large values of $m$, we now show that the maximum value in $Y^{(\theta)}$ is an arbitrarily good approximation to $\left\|\left(T^{(\theta)}\right)^K\right\|^{\frac{1}{K}}$ (in the statistical sense).
\begin{theorem}\label{thm:optimality}
  For any nonzero matrix $T$, $\epsilon>0$, and $\delta <1$, there exist $M, K\in\mathbb{N}$ such that for any $m>M$, if one chooses $m$ points, $x_j$, uniformly at random from $\{x\in\mathbb{R}^n, \|x\| = 1\}$, then $Y = \left\{\left\|T^K x_1\right\| ,\left\|T^K x_2\right\|, \ldots, \left\|T^K x_m\right\|\right\}$ satisfies
  \begin{equation}
  P\left(\left|\rho(T) - \max(Y)^{\frac{1}{K}} \right| <\epsilon\right) > 1 - \delta.
  \end{equation}
  \end{theorem}
\begin{proof}

{According to Gelfand's theorem, there exists $L\in\mathbb{N}$ such that $\forall{\ell>L}$, $\left|\rho(T) - \displaystyle\sup_{x: \|x\|=1}  \|T^{\ell}x\|^{\frac{1}{\ell}}\right| < \frac{\epsilon}{2}$. Take any $K \ge L$ and let $\tilde{\epsilon} = \frac{\epsilon}{2\|T^K\|^{\frac{1}{K}}}$.} Since $\mathbb{R}^n$ is finite-dimensional, there exists an $x^*\in\mathbb{R}^n,\;\|x^*\|=1$ such that $\displaystyle\sup_{x: \|x\|=1}  \|T^{K}x\| = \|T^{K}x^*\|$. Denote the volume of the surface of the $n$-dimensional sphere of unit radius around the origin in $\mathbb{R}^n$ by $C_{\text{tot}}$, and the volume of the region on this sphere within radius $\tilde{\epsilon}^K$ of $x^*$ by $C_{\tilde{\epsilon}, K}$.

  Given $\delta < 1$, let $M\in\mathbb{N}$ satisfy $M \geq \frac{\log(\delta)}{\log\left(1-\frac{C_{\tilde{\epsilon}, K}}{C_{\text{tot}}}\right)}$.  Then,
\begin{equation}
P\left(\|x^*-x_i\|>\tilde{\epsilon}^K\text{ for all }i\right) <  \left(1-\frac{C_{\tilde{\epsilon}, K}}{C_{\text{tot}}}\right)^M \le \delta.
\end{equation}
Thus, with probability of at least $1-\delta$, we expect at least one point from a selection of $m>M$ points uniformly distributed on the sphere of radius unit radius to be in $C_{\tilde{\epsilon},K}$.  Let that point be $x_r$, giving
\begin{align}
  \left\|T^K x^*\right\|^\frac{1}{K}-\left\|T^K x_r\right\|^\frac{1}{K}&\leq\left(\left\|T^K x^*\right\|-\left\|T^K x_r\right\|\right)^{\frac{1}{K}} \\
  &\leq \left(\left\|T^K(x^*-x_r)\right\|\right)^{\frac{1}{K}}\\
  &\le\left\|T^K\right\|^{\frac{1}{K}}\|x^*-x_r\|^{\frac{1}{K}}\le\|T^K\|^{\frac{1}{K}}\tilde{\epsilon} = \frac{\epsilon}{2}
\end{align}
using Lemma~\ref{lem:lem} (see Supplementary Materials) and the reverse triangle inequality. {Finally, by the triangle inequality, with probability of at least $1-\delta$, we have:}
\begin{equation}
\left|\rho(T) - \text{max}(Y)^{\frac{1}{K}}\right|  \le  \left|\rho(T) - \sup_{x: \|x\|=1}  \|T^{K}x\|^{\frac{1}{K}}\right| + \left|\sup_{x: \|x\|=1}  \|T^{K}x\|^{\frac{1}{K}} - \text{max}(Y)^{\frac{1}{K}}\right| \le \epsilon.
\end{equation}
\end{proof}

\begin{remark}\label{thm:learnability}
According to~\cite{loukas2019graph}, since the optimal interface values are Turing-computable, if the depth of the GNN is at least the diameter of the graph, and a TAGConv layer followed by a feature encoder is Turing-complete, the optimal interface values can be learned.  In our setting, for a problem on a structured grid of size $N\times N$ with two identical rectangular subdomains, this implies that the GNN will be able to learn the optimal interface values given, if and only the GNN has depth at least $2N$, has deep enough feature encoders, and the width of the layers is unbounded.
\end{remark}
\begin{remark}\label{thm:effectiveness}

Theorem~\ref{thm:optimality} guarantees convergence of the loss function to the spectral radius, in the limits of many samples and many stationary iterations.  To the best of our knowledge, such a guarantee is not known for the previous loss functions used in the area~\cite{luz2020learning, greenfeld2019learning}. Moreover, there are substantial improvements in the numerical results using the new loss function in comparison to that of~\cite{luz2020learning}, as shown in Figure~\ref{fig:better_loss}.

\end{remark}
\vspace{0.1em}
\begin{theorem}\label{thm:complexity}
  Assuming bounded subdomain size, the time complexity to evaluate the optimal Schwarz parameters using our method is $O(n)$, where $n$ is the number of nodes in the grid.
\end{theorem}
\begin{proof}
Given bounded Lloyd subdomain size and fixed number of Lloyd aggregation cycles, subdomain generation has $O(n)$ time complexity~\cite{bell2008algebraic} (see the Supplementary Material). To evaluate each TAGConv layer, one computes $y = \sum_{\ell=1}^{L}G_{\ell}x_{\ell} + b\mathbf{1}_{n}$, where $L$ is the number of node features, $G_{\ell} \in  \mathbb{R}^{n\times n}$ is the graph filter, $b$ is a learnable bias, and $x_{\ell}\in \mathbb{R}^{n}$ are the node features. Moreover, the graph filter is a polynomial in the adjacency matrix $M$ of the graph: $G_{\ell} = \sum_{j=0}^{J}g_{\ell,j}M^{j}$ where $J$ is a constant and $g_{\ell,j}$ are the coefficients of filter polynomial. Since the graph has bounded node degrees, it implies that $M$ is sparse and the action of $M^j$ has $O(n)$ cost, and therefore, the full TAGConv evaluation also has $O(n)$ cost. Moreover, the cost of edge feature and the feature networks are $O(n)$, resulting in overall $O(n)$ cost.
\end{proof}

\subsection{Training Details}\label{sec:training}

We use a graph neural network based on four TAGConv layers and ResNet node feature encoders consisting of eight blocks after each TAGConv layer; see Supplementary Materials for more details on the structure of the GNN, and Section~\ref{sec:ablation} for an ablation study.  The training set in our study consists of 1000 unstructured grids with piecewise linear finite elements and with grids ranging from 90--850 nodes (and an average of 310 nodes). The grids are generated by choosing either a regular grid (randomly selected 60\% of the time) or a randomly generated convex polygon; pygmsh~\cite{Schlmer_pygmsh_A_Python} is used to generate the mesh on the polygon interior. A sample of the training grids is shown in Figure~\ref{fig:train_exs}. We train the GNN for four epochs with a mini batch size of 25 using the ADAM optimizer~\cite{kingma2014adam} with a fixed learning rate of $10^{-4}$. For the numerical evaluation of the loss function~(\ref{eq:numeric_loss}) we use $K = 4$ iterations and $m = 500$ samples. The code\footnote[1]{All code and data for this paper is at \url{https://github.com/compdyn/learning-oras} (MIT licensed).} was implemented using PyTorch Geometric~\cite{Fey_Lenssen_2019}, PyAMG~\cite{BeOlSc2022}, and NetworkX~\cite{hagberg2008exploring}. All training was performed on an 8-core i9 Macbook Pro using CPU only.
\begin{figure}
  \includegraphics[width = \textwidth]{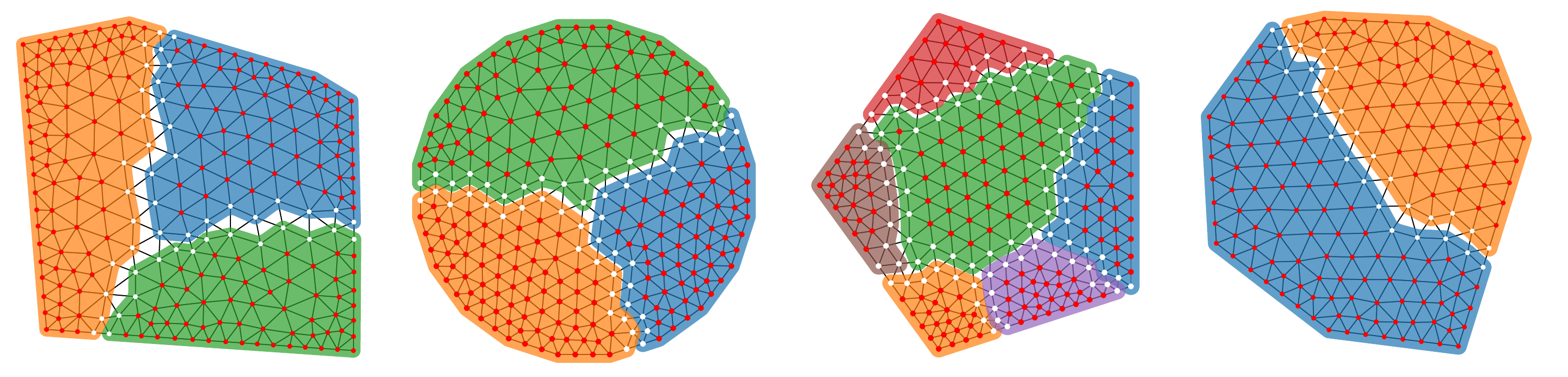}
  \caption{Example grids from the training set.}\label{fig:train_exs}
\end{figure}

Moreover, we train two special networks for use in Section~\ref{sec:twosubd}. The first is a ``Brute Force'' network, which is a single-layer neural network without an activation function, trained only on a structured $10\times 10$ grid with two identical subdomains using the ADAM optimizer. The purpose of training this is to obtain the optimal interface values as a benchmark for comparison, including against our method. The second is the same network used for our method, MLORAS, but overtrained on only the problem in Section~\ref{sec:twosubd}, to understand the learning capabilities of the method and choice of GNN architecture.

\section{Results}

\subsection{Two-subdomain structured grids}\label{sec:twosubd}

We first consider rectangular structured grids with two subdomains. Although restrictive, these problems allow us to directly compare to existing OSM parameters, which are only available for structured grids and exactly two subdomains. We follow~\cite{st2007optimized} and consider the Helmholtz problem~\eqref{eq:Helmholtz} on the unit square with $\eta = 1$ and homogeneous Dirichlet boundary conditions. We discretize the problem on an $N \times N$ rectangular grid with $h = 1/(N+1)$ grid spacing, using the standard five-point finite difference stencil.


\begin{figure}
  \setlength{\tabcolsep}{0pt}
  \begin{tabular}{p{0.275\textwidth} p{0.725\textwidth}}
  \vspace{15pt} \includegraphics[width=0.275\textwidth]{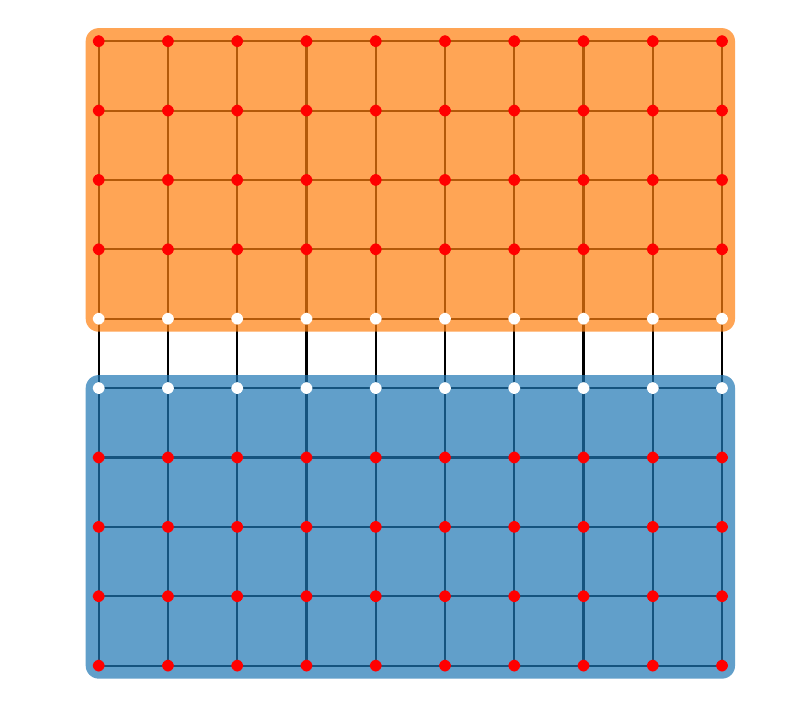} &
  \vspace{0pt} \includegraphics[width=0.725\textwidth]{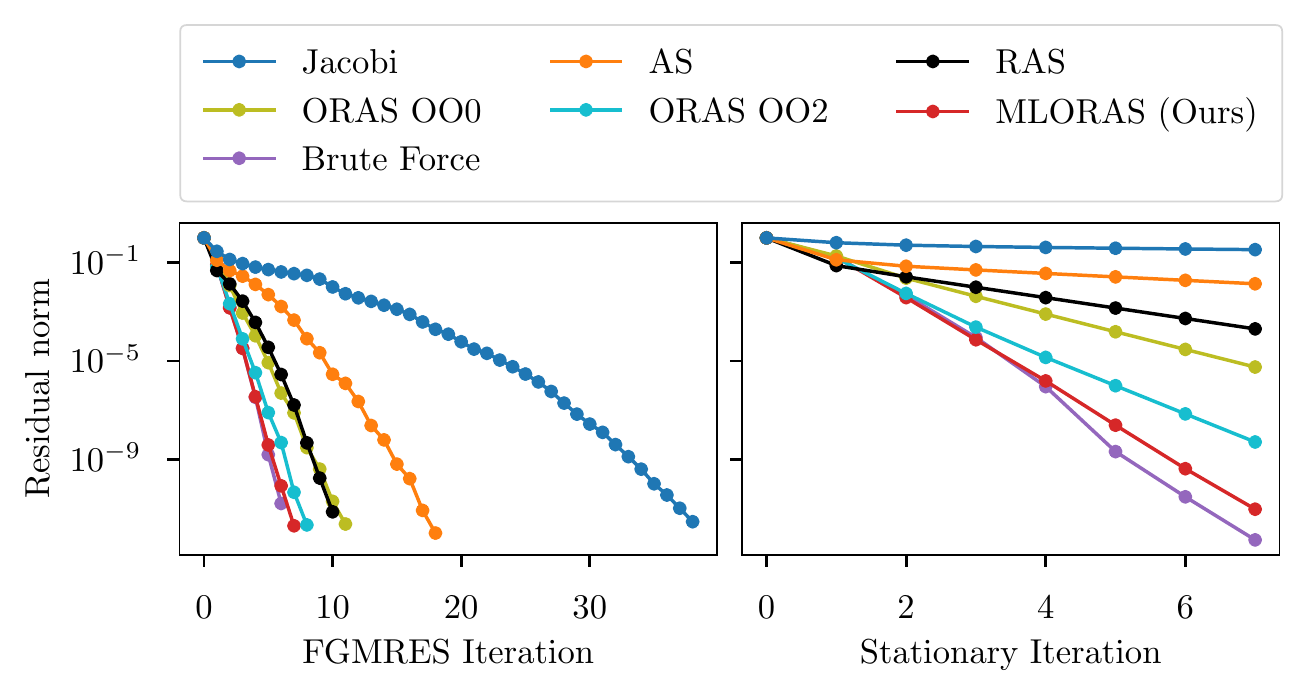}
  \end{tabular}
  \caption{Results for the Helmholtz problem on a $10\times10$ structured grid (left) with two identical subdomains, with convergence plots for the methods used as preconditioners for FGMRES (center) and as stationary algorithms (right).}\label{fig:structs_results10}
\end{figure}

Figure~\ref{fig:structs_results10} shows the results of different methods on a $10\times10$ structured grid with two identical subdomains. We see that the overtrained MLORAS network learns interface parameters that outperform the Restrictive Additive Schwarz (RAS)~\cite{cai1999restricted} method; more significantly, MLORAS also outperforms the existing Optimized-RAS (ORAS) methods that were analytically derived for this specific problem (zeroth-order OO0 and second-order OO2 from~\cite{st2007optimized}).
Moreover, in these results, the MLORAS network almost reaches the performance of the ``Brute Force'' network (obtained by directly optimizing all interface values for this single structured grid), indicating that using a GNN to encode the optimal interface values does not significantly restrict the search space.

\begin{figure}
  \centering
    \includegraphics[width = 1\textwidth,,trim=0 10 0 10]{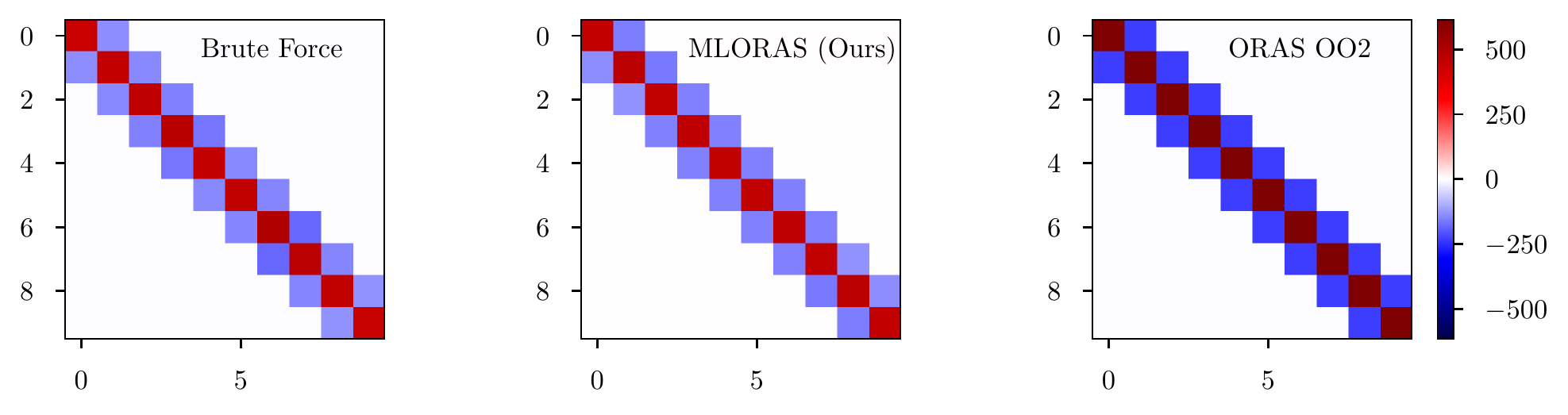}
  \caption{Interface values for the $10\times10$ structured grid with two identical subdomains. From left to right: brute force optimization of the interface values, MLORAS, and second-order ORAS (OO2)~\cite{st2007optimized}.}\label{fig:heatmap}
\end{figure}

To understand the performance of the methods in more detail, Figure~\ref{fig:heatmap} plots the interface values output by the brute force optimization, the MLORAS network, and the OO2 ORAS algorithm. This shows that the MLORAS network is choosing interface values very close to those directly optimized by the brute force method, unlike those selected by the OO2 method.








\subsection{Unstructured grids}\label{sec:result_uns}

To evaluate the performance of the MLORAS network on unstructured grids, we consider 16 unstructured triangular grids in 2D with sizes ranging from about 90 to 40\,000 nodes. These grids are defined on convex subsets of $(0,1)\times(0,1)$; we solve the Helmholtz problem~\eqref{eq:Helmholtz} on these domains with $\eta = 1$ and homogeneous Dirichlet boundary conditions, discretized with piecewise-linear finite elements.

Example convergence plots are shown in Figure~\ref{fig:unst_scale}, where our method (MLORAS) is compared to RAS (Restricted Additive Schwarz~\cite{cai1999restricted}), AS (Additive Schwarz), and Jacobi methods as a preconditioner for FGMRES\@. Here we see that the MLORAS network is able to learn optimized interface parameters for unstructured grids that outperform RAS, and that MLORAS can scale to problems that are much larger than those in the training set, which are all below $n = 1000$ nodes. Importantly, MLORAS retains an advantage over RAS even as the grid size increases.

\begin{figure}
  \centering
\hspace{1.4cm}\includegraphics[width = 0.175\textwidth,trim=0 0 0 40]{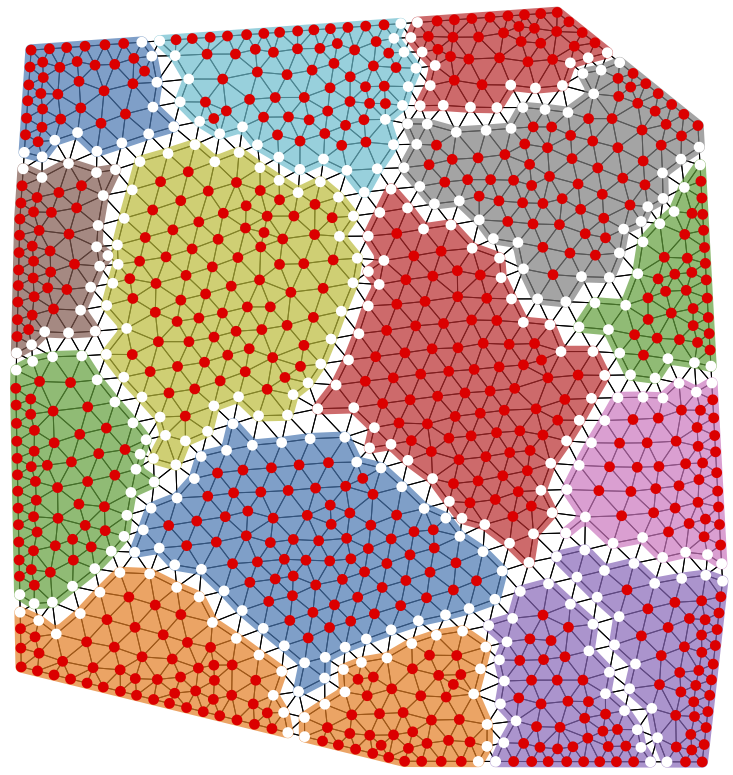}
\hspace{3.175cm}\includegraphics[width = 0.25\textwidth,trim=0 0 0 40]{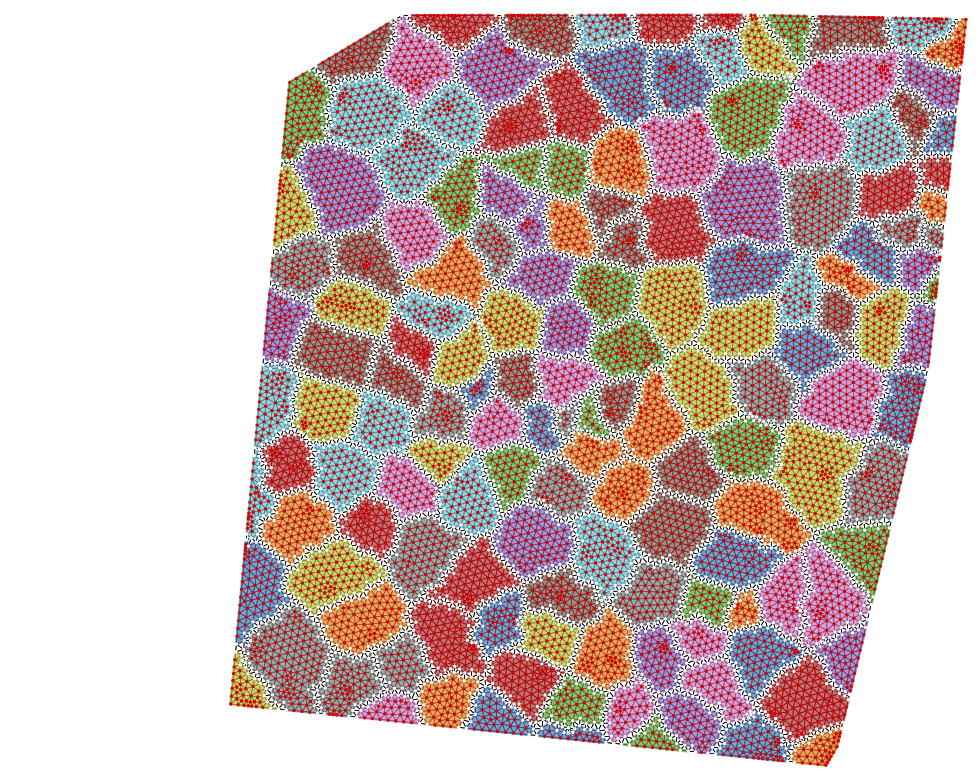}
 \includegraphics[width = \textwidth,trim=0 10 0 0]{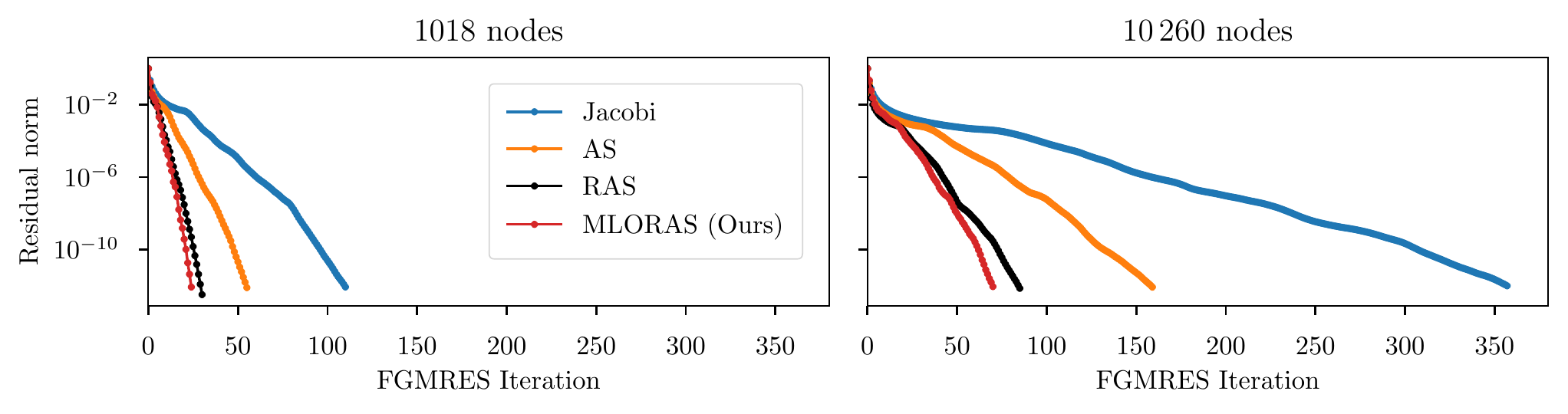}
  \caption{Example convergence on smaller (left) and larger (right) unstructured grids.}\label{fig:unst_scale}
\end{figure}

\begin{figure}
  \centering
    \includegraphics[width = 1\textwidth,trim=0 5 0 5]{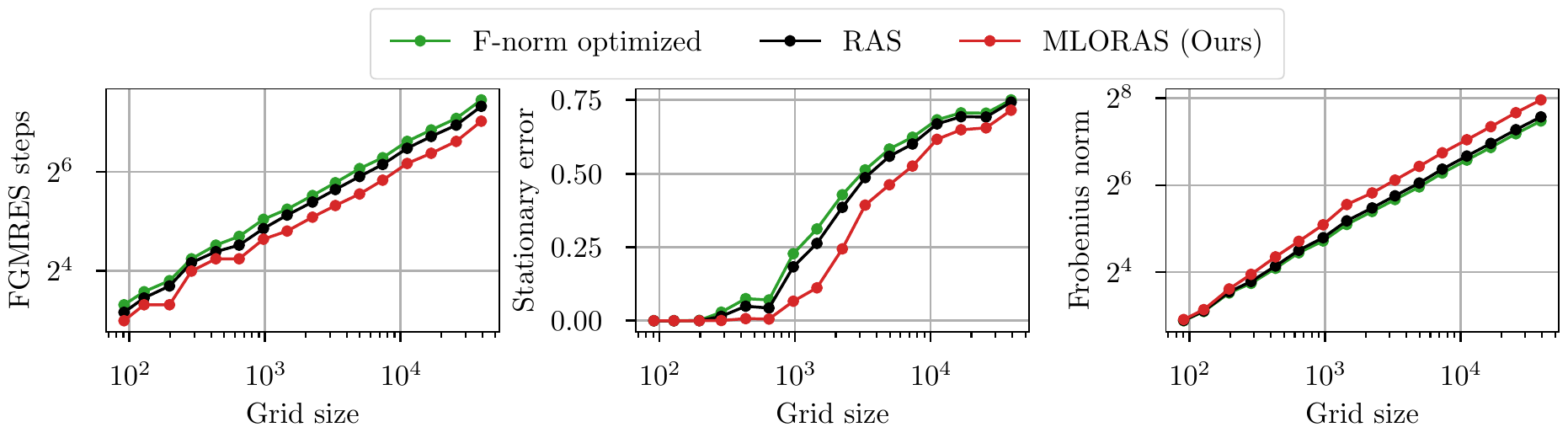}
  \caption{Convergence on all unstructured grids in the testing set. Left: the number of preconditioned FGMRES steps required to solve the problem to within a relative error of $10^{-12}$. Center: error reduction by the stationary iteration after 10 iterations. Right: Frobenius norm for each method on each test problem. }\label{fig:better_loss}
\end{figure}

Figure~\ref{fig:better_loss} shows the performance of three different methods across all unstructured test grids. The three methods are: (1) ``F-norm optimized'': the same network as MLORAS, but trained to directly optimize the Frobenius norm of the error propagation matrix, $\|A\|_F$, (2) the RAS~\cite{cai1999restricted} method, and (3) the MLORAS method. We do not show ORAS results here, as it cannot be applied to unstructured grids.
Figure~\ref{fig:better_loss} reveals two important facts. First, MLORAS consistently outperforms RAS over the entire testing set, both as an FGMRES preconditioner and as a stationary algorithm, and this advantage is maintained even for large grids. Second, we see that the Frobenius norm is indeed a worse choice for the loss function than our new loss in~\eqref{eq:loss_func}. We see this from the fact that MLORAS has a worse Frobenius norm than either of the other two methods, but it has the best convergence rate. In addition, when we explicitly optimize the Frobenius norm (the ``F-norm optimized'' method), we see that we do obtain the lowest Frobenius norm, but this translates to the worst convergence rate.

\subsection{Ablation study}\label{sec:ablation}

To understand the impact of the network structure on performance, we conduct an ablation study by varying the ResNet block length and the number of TAGConv layers. In each case, we train a network on five different training sets, each consisting of 1000 grids generated as described in Section~\ref{sec:training}. The trained networks are tested on 50 unstructured grids, each with 2400 to 2600 nodes. For each architecture, the mean performance is computed, together with error bars from the fivefold repetition, with results shown in Figure~\ref{fig:ablation}. We see that a higher residual block length is always better, but that the optimal number of TAGConv layers is 4, and using more layers actually decreases performance. This phenomenon has been observed in other contexts (e.g.,~\cite{zhao2019pairnorm} and~\cite{chen2020measuring}) and can be attributed to over-smoothing in GNNs. The use of residual blocks in our architecture is, thus, important to allow us to increase network depth without needing to increase the number of GNN layers.

\begin{figure}
  \centering
  \includegraphics[width = \textwidth,trim=0 10 0 10]{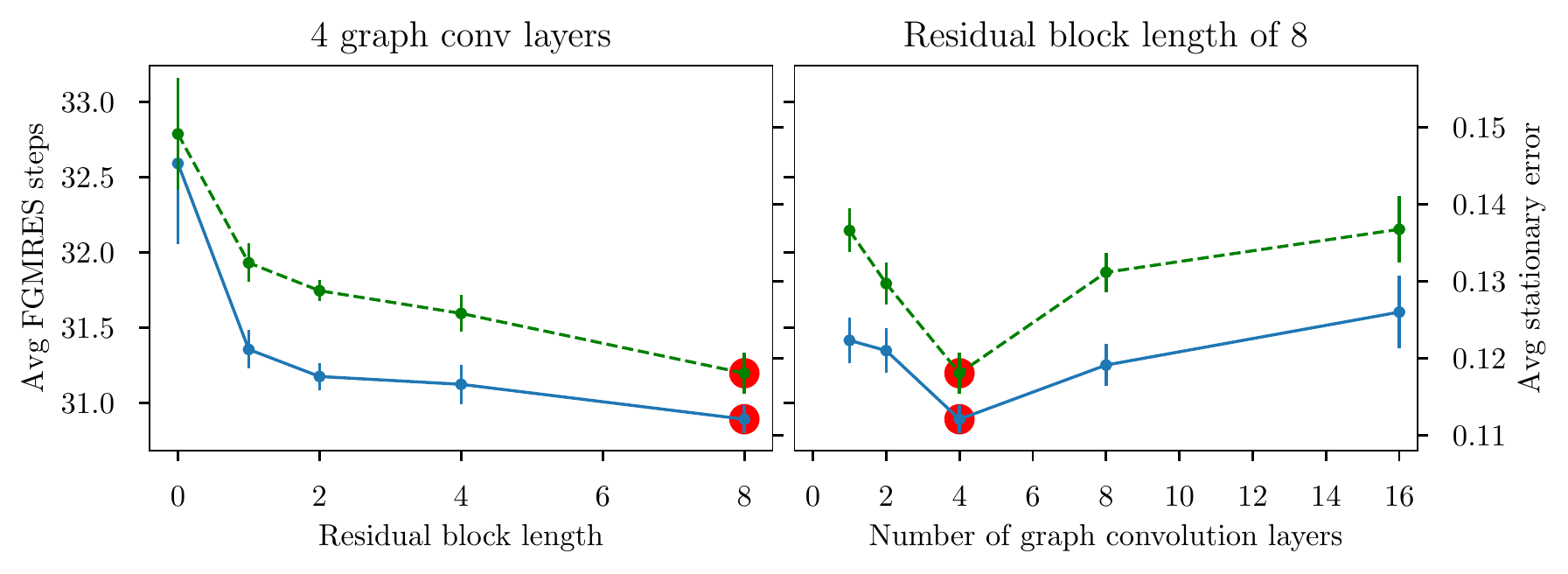}
  \caption{Ablation study results. The left panel varies the residual block length while keeping the number of TAGConv layers fixed at 4. The right panel varies the number of TAGConv layers while keeping the residual block length fixed at 8. The solid blue lines (left axis) show the average number of FGMRES steps needed to reduce the relative error below $10^{-12}$, while the dashed green lines (right axis) show the average stationary algorithm error reduction after 10 iterations. The red circles mark the results for the network with the best performance (residual block length of 8 and 4 TAGConv layers), which was used for all other studies in this paper. Error bars show one standard error of the mean.}\label{fig:ablation}
\end{figure}

\section{Conclusion}\label{sec:conclusion}

We propose an unsupervised learning method based on Graph Convolutional Neural Networks (GCNNs) for extending optimized restricted additive Schwarz (ORAS) methods to multiple subdomains and unstructured grid cases. Our method is trained with a novel loss function, stochastically minimizing the spectral radius of the error propagation operator obtained using the learned interface values. The time complexity of evaluating the loss function, as well as obtaining the interface values using our neural network are all linear in problem size. Moreover, the  proposed method is able to outperform ORAS, both as a stationary algorithm and preconditioner for FGMRES on structured grids with two subdomains, as considered in the conventional ORAS literature. On more general cases, such as unstructured grids with arbitrary subdomains of bounded size, our method outperforms RAS consistently, both as a stationary algorithm and preconditioner for FGMRES. The main limitations of the current work are that the method was studied for two PDE cases, namely the Helmholtz problem in the main paper and the non-uniform Poisson problem in Appendix~\ref{sec:poisson}. We also defer the study of nonlinear or time-dependent PDEs to future work.

\section*{Acknowledgement}
The work of S.P.M. was partially supported by an NSERC Discovery Grant. The authors thank the referees for their insightful comments. The authors have no competing interests to declare.
\clearpage
\bibliographystyle{unsrtnat}
\bibliography{paper-oras-neurips}

\begin{thebibliography}{40}
\providecommand{\natexlab}[1]{#1}
\providecommand{\url}[1]{\texttt{#1}}
\expandafter\ifx\csname urlstyle\endcsname\relax
  \providecommand{\doi}[1]{doi: #1}\else
  \providecommand{\doi}{doi: \begingroup \urlstyle{rm}\Url}\fi

\bibitem[Toselli and Widlund(2005)]{MR2104179}
Andrea Toselli and Olof Widlund.
\newblock \emph{Domain decomposition methods---algorithms and theory},
  volume~34 of \emph{Springer Series in Computational Mathematics}.
\newblock Springer-Verlag, Berlin, 2005.
\newblock ISBN 3-540-20696-5.
\newblock \doi{10.1007/b137868}.

\bibitem[Quarteroni and Valli(1999)]{MR1857663}
Alfio Quarteroni and Alberto Valli.
\newblock \emph{Domain decomposition methods for partial differential
  equations}.
\newblock Numerical Mathematics and Scientific Computation. The Clarendon
  Press, Oxford University Press, New York, 1999.
\newblock ISBN 0-19-850178-1.
\newblock Oxford Science Publications.

\bibitem[Dolean et~al.(2015)Dolean, Jolivet, and Nataf]{MR3450068}
Victorita Dolean, Pierre Jolivet, and Fr\'{e}d\'{e}ric Nataf.
\newblock \emph{An introduction to domain decomposition methods}.
\newblock Society for Industrial and Applied Mathematics (SIAM), Philadelphia,
  PA, 2015.
\newblock ISBN 978-1-611974-05-8.
\newblock \doi{10.1137/1.9781611974065.ch1}.

\bibitem[Gander et~al.(2000)Gander, Halpern, and Nataf]{GHN_optimized_2000}
M.J. Gander, L.~Halpern, and F.~Nataf.
\newblock Optimized {S}chwarz methods.
\newblock In \emph{Proceedings of the 12th International Conference on Domain
  Decomposition}, pages 15--27. ddm.org, 2000.

\bibitem[Gander and Kwok(2011)]{MR2867648}
Martin~J. Gander and Felix Kwok.
\newblock Optimal interface conditions for an arbitrary decomposition into
  subdomains.
\newblock In \emph{Domain decomposition methods in science and engineering
  {XIX}}, volume~78 of \emph{Lect. Notes Comput. Sci. Eng.}, pages 101--108.
  Springer, Heidelberg, 2011.
\newblock \doi{10.1007/978-3-642-11304-8\_9}.

\bibitem[Cai and Sarkis(1999)]{cai1999restricted}
Xiao-Chuan Cai and Marcus Sarkis.
\newblock A restricted additive {S}chwarz preconditioner for general sparse
  linear systems.
\newblock \emph{Siam Journal on Scientific Computing}, 21\penalty0
  (2):\penalty0 792--797, 1999.

\bibitem[St-Cyr et~al.(2007)St-Cyr, Gander, and Thomas]{st2007optimized}
Amik St-Cyr, Martin~J Gander, and Stephen~J Thomas.
\newblock Optimized multiplicative, additive, and restricted additive {S}chwarz
  preconditioning.
\newblock \emph{SIAM Journal on Scientific Computing}, 29\penalty0
  (6):\penalty0 2402--2425, 2007.

\bibitem[Magoul\`es et~al.(2017)Magoul\`es, Szyld, and Venet]{MR3679933}
Fr\'{e}d\'{e}ric Magoul\`es, Daniel~B. Szyld, and C\'{e}dric Venet.
\newblock Asynchronous optimized {S}chwarz methods with and without overlap.
\newblock \emph{Numer. Math.}, 137\penalty0 (1):\penalty0 199--227, 2017.
\newblock ISSN 0029-599X.
\newblock \doi{10.1007/s00211-017-0872-z}.

\bibitem[Raissi et~al.(2019)Raissi, Perdikaris, and
  Karniadakis]{raissi2019physics}
Maziar Raissi, Paris Perdikaris, and George~E Karniadakis.
\newblock Physics-informed neural networks: A deep learning framework for
  solving forward and inverse problems involving nonlinear partial differential
  equations.
\newblock \emph{Journal of Computational Physics}, 378:\penalty0 686--707,
  2019.

\bibitem[Li et~al.(2020{\natexlab{a}})Li, Kovachki, Azizzadenesheli, Liu,
  Bhattacharya, Stuart, and Anandkumar]{li2020fourier}
Zongyi Li, Nikola Kovachki, Kamyar Azizzadenesheli, Burigede Liu, Kaushik
  Bhattacharya, Andrew Stuart, and Anima Anandkumar.
\newblock Fourier neural operator for parametric partial differential
  equations.
\newblock \emph{arXiv preprint arXiv:2010.08895}, 2020{\natexlab{a}}.

\bibitem[Greenfeld et~al.(2019)Greenfeld, Galun, Basri, Yavneh, and
  Kimmel]{greenfeld2019learning}
Daniel Greenfeld, Meirav Galun, Ronen Basri, Irad Yavneh, and Ron Kimmel.
\newblock Learning to optimize multigrid {PDE} solvers.
\newblock In \emph{International Conference on Machine Learning}, pages
  2415--2423. PMLR, 2019.

\bibitem[Luz et~al.(2020)Luz, Galun, Maron, Basri, and Yavneh]{luz2020learning}
Ilay Luz, Meirav Galun, Haggai Maron, Ronen Basri, and Irad Yavneh.
\newblock Learning algebraic multigrid using graph neural networks.
\newblock In \emph{International Conference on Machine Learning}, pages
  6489--6499. PMLR, 2020.

\bibitem[Taghibakhshi et~al.(2021)Taghibakhshi, MacLachlan, Olson, and
  West]{taghibakhshi2021optimization}
Ali Taghibakhshi, Scott MacLachlan, Luke Olson, and Matthew West.
\newblock Optimization-based algebraic multigrid coarsening using reinforcement
  learning.
\newblock \emph{Advances in Neural Information Processing Systems}, 34, 2021.

\bibitem[Heinlein et~al.(2021)Heinlein, Klawonn, Lanser, and
  Weber]{heinlein2021combining}
Alexander Heinlein, Axel Klawonn, Martin Lanser, and Janine Weber.
\newblock Combining machine learning and domain decomposition methods for the
  solution of partial differential equations---a review.
\newblock \emph{GAMM-Mitteilungen}, 44\penalty0 (1):\penalty0 e202100001, 2021.

\bibitem[Heinlein et~al.(2019)Heinlein, Klawonn, Lanser, and
  Weber]{heinlein2019machine}
Alexander Heinlein, Axel Klawonn, Martin Lanser, and Janine Weber.
\newblock Machine learning in adaptive domain decomposition
  methods---predicting the geometric location of constraints.
\newblock \emph{SIAM Journal on Scientific Computing}, 41\penalty0
  (6):\penalty0 A3887--A3912, 2019.

\bibitem[Li et~al.(2019)Li, Tang, Wu, and Liao]{li2019d3m}
Ke~Li, Kejun Tang, Tianfan Wu, and Qifeng Liao.
\newblock {D3M}: A deep domain decomposition method for partial differential
  equations.
\newblock \emph{IEEE Access}, 8:\penalty0 5283--5294, 2019.

\bibitem[Li et~al.(2020{\natexlab{b}})Li, Xiang, and Xu]{li2020deep}
Wuyang Li, Xueshuang Xiang, and Yingxiang Xu.
\newblock Deep domain decomposition method: Elliptic problems.
\newblock In \emph{Mathematical and Scientific Machine Learning}, pages
  269--286. PMLR, 2020{\natexlab{b}}.

\bibitem[Mercier et~al.(2021)Mercier, Gratton, and Boudier]{mercier2021coarse}
Valentin Mercier, Serge Gratton, and Pierre Boudier.
\newblock A coarse space acceleration of deep-{DDM}.
\newblock \emph{arXiv preprint arXiv:2112.03732}, 2021.

\bibitem[Wu et~al.(2020)Wu, Pan, Chen, Long, Zhang, and
  Philip]{wu2020comprehensive}
Zonghan Wu, Shirui Pan, Fengwen Chen, Guodong Long, Chengqi Zhang, and S~Yu
  Philip.
\newblock A comprehensive survey on graph neural networks.
\newblock \emph{IEEE Transactions on Neural Networks and Learning Systems},
  2020.

\bibitem[Bruna et~al.(2013)Bruna, Zaremba, Szlam, and LeCun]{bruna2013spectral}
Joan Bruna, Wojciech Zaremba, Arthur Szlam, and Yann LeCun.
\newblock Spectral networks and locally connected networks on graphs.
\newblock \emph{arXiv preprint arXiv:1312.6203}, 2013.

\bibitem[Defferrard et~al.(2016)Defferrard, Bresson, and
  Vandergheynst]{defferrard2016convolutional}
Micha{\"e}l Defferrard, Xavier Bresson, and Pierre Vandergheynst.
\newblock Convolutional neural networks on graphs with fast localized spectral
  filtering.
\newblock \emph{arXiv preprint arXiv:1606.09375}, 2016.

\bibitem[Kipf and Welling(2016)]{kipf2016semi}
Thomas~N Kipf and Max Welling.
\newblock Semi-supervised classification with graph convolutional networks.
\newblock \emph{arXiv preprint arXiv:1609.02907}, 2016.

\bibitem[Gilmer et~al.(2017)Gilmer, Schoenholz, Riley, Vinyals, and
  Dahl]{gilmer2017neural}
Justin Gilmer, Samuel~S Schoenholz, Patrick~F Riley, Oriol Vinyals, and
  George~E Dahl.
\newblock Neural message passing for quantum chemistry.
\newblock In \emph{International Conference on Machine Learning}, pages
  1263--1272. PMLR, 2017.

\bibitem[Battaglia et~al.(2018)Battaglia, Hamrick, Bapst, Sanchez-Gonzalez,
  Zambaldi, Malinowski, Tacchetti, Raposo, Santoro, Faulkner,
  et~al.]{battaglia2018relational}
Peter~W Battaglia, Jessica~B Hamrick, Victor Bapst, Alvaro Sanchez-Gonzalez,
  Vinicius Zambaldi, Mateusz Malinowski, Andrea Tacchetti, David Raposo, Adam
  Santoro, Ryan Faulkner, et~al.
\newblock Relational inductive biases, deep learning, and graph networks.
\newblock \emph{arXiv preprint arXiv:1806.01261}, 2018.

\bibitem[Du et~al.(2017)Du, Zhang, Wu, Moura, and Kar]{du2017topology}
Jian Du, Shanghang Zhang, Guanhang Wu, Jos{\'e}~MF Moura, and Soummya Kar.
\newblock Topology adaptive graph convolutional networks.
\newblock \emph{arXiv preprint arXiv:1710.10370}, 2017.

\bibitem[Wang et~al.(2019)Wang, Dang, Hu, Fua, and
  Salzmann]{wang2019backpropagation}
Wei Wang, Zheng Dang, Yinlin Hu, Pascal Fua, and Mathieu Salzmann.
\newblock Backpropagation-friendly eigendecomposition.
\newblock \emph{Advances in Neural Information Processing Systems}, 32, 2019.

\bibitem[Karypis and Kumar(1998)]{karypis1998fast}
George Karypis and Vipin Kumar.
\newblock A fast and high quality multilevel scheme for partitioning irregular
  graphs.
\newblock \emph{SIAM Journal on scientific Computing}, 20\penalty0
  (1):\penalty0 359--392, 1998.

\bibitem[Bell(2008)]{bell2008algebraic}
William~N Bell.
\newblock \emph{Algebraic multigrid for discrete differential forms}.
\newblock PhD thesis, University of Illinois at Urbana-Champaign, 2008.

\bibitem[Lloyd(1982)]{lloyd1982least}
Stuart Lloyd.
\newblock Least squares quantization in {PCM}.
\newblock \emph{IEEE Transactions on Information Theory}, 28\penalty0
  (2):\penalty0 129--137, 1982.

\bibitem[Box(1958)]{box1958note}
George~EP Box.
\newblock A note on the generation of random normal deviates.
\newblock \emph{Ann. Math. Statist.}, 29:\penalty0 610--611, 1958.

\bibitem[Katrutsa et~al.(2020)Katrutsa, Daulbaev, and
  Oseledets]{KATRUTSA2020112524}
Alexandr Katrutsa, Talgat Daulbaev, and Ivan Oseledets.
\newblock Black-box learning of multigrid parameters.
\newblock \emph{Journal of Computational and Applied Mathematics},
  368:\penalty0 112524, 2020.
\newblock ISSN 0377-0427.
\newblock \doi{https://doi.org/10.1016/j.cam.2019.112524}.

\bibitem[Loukas(2019)]{loukas2019graph}
Andreas Loukas.
\newblock What graph neural networks cannot learn: depth vs width.
\newblock \emph{arXiv preprint arXiv:1907.03199}, 2019.

\bibitem[Schl{\"o}mer(2021)]{Schlmer_pygmsh_A_Python}
Nico Schl{\"o}mer.
\newblock {pygmsh: A Python frontend for Gmsh}, 2021.
\newblock URL \url{https://github.com/nschloe/pygmsh}.
\newblock (GPL-3.0 License).

\bibitem[Kingma and Ba(2014)]{kingma2014adam}
Diederik~P Kingma and Jimmy Ba.
\newblock Adam: A method for stochastic optimization.
\newblock \emph{arXiv preprint arXiv:1412.6980}, 2014.

\bibitem[Fey and Lenssen(2019)]{Fey_Lenssen_2019}
Matthias Fey and Jan~E. Lenssen.
\newblock Fast graph representation learning with {PyTorch Geometric}.
\newblock In \emph{ICLR Workshop on Representation Learning on Graphs and
  Manifolds}, 2019.
\newblock (code is MIT licensed).

\bibitem[Bell et~al.(2022)Bell, Olson, and Schroder]{BeOlSc2022}
Nathan Bell, Luke~N. Olson, and Jacob Schroder.
\newblock {PyAMG}: Algebraic multigrid solvers in python.
\newblock \emph{Journal of Open Source Software}, 7\penalty0 (72):\penalty0
  4142, 2022.
\newblock \doi{10.21105/joss.04142}.
\newblock URL \url{https://doi.org/10.21105/joss.04142}.
\newblock (code is MIT licensed).

\bibitem[Hagberg et~al.(2008)Hagberg, Swart, and S~Chult]{hagberg2008exploring}
Aric Hagberg, Pieter Swart, and Daniel S~Chult.
\newblock Exploring network structure, dynamics, and function using networkx.
\newblock Technical report, Los Alamos National Lab.(LANL), Los Alamos, NM
  (United States), 2008.
\newblock (code is BSD licensed).

\bibitem[Zhao and Akoglu(2019)]{zhao2019pairnorm}
Lingxiao Zhao and Leman Akoglu.
\newblock Pairnorm: Tackling oversmoothing in {GNNs}.
\newblock \emph{arXiv preprint arXiv:1909.12223}, 2019.

\bibitem[Chen et~al.(2020)Chen, Lin, Li, Li, Zhou, and Sun]{chen2020measuring}
Deli Chen, Yankai Lin, Wei Li, Peng Li, Jie Zhou, and Xu~Sun.
\newblock Measuring and relieving the over-smoothing problem for graph neural
  networks from the topological view.
\newblock In \emph{Proceedings of the AAAI Conference on Artificial
  Intelligence}, volume~34, pages 3438--3445, 2020.

\bibitem[Cormen et~al.(2001)Cormen, Leiserson, Rivest, and
  Stein]{cormen2001introduction}
Thomas~H Cormen, Charles~E Leiserson, Ronald~L Rivest, and Clifford Stein.
\newblock \emph{Introduction to algorithms (Second edition)}.
\newblock MIT Press, 2001.

\end{thebibliography}

\newpage

\begin{appendices}
\maketitle

\section{Useful lemma}

This lemma is useful in the proof of Theorem~\ref{thm:optimality}.  While its proof is quite simple, we are not aware of the result in the literature, and include it here for completeness.
\begin{lemma}\label{lem:lem}
For $x,y\in\mathbb{R}$, with $0\le y \le x$ and any $K\in\mathbb{N}$, $x^{\frac{1}{K}}-y^{\frac{1}{K}}\le(x-y)^{\frac{1}{K}}$
\end{lemma}
\begin{proof}
Since $y\le x$, the binomial theorem gives that
\begin{equation}
  x\le (x-y) + y + \sum\limits_{i=1}^{K-1}\binom{K}{i}y^{\frac{i}{K}}(x-y)^{\frac{K-i}{K}}=(y^{\frac{1}{K}}+(x-y)^{\frac{1}{K}})^K.
\end{equation}
Taking the $K^{\text{th}}$ root of both sides and rearranging gives the stated result.
\end{proof}

\section{Solution plots}

We have compared the solution plot of our method with RAS for the Holmholtz problem. We consider a $100\times 100$ structured grid on the $(0,1)\times(0,1)$ domain, and consider the following true solution for the problem: $u^{*} = \text{sin}(8\pi x) + \text{sin}(8\pi y)$. We then start with an initial random guess with $L_2$ norm of 1. We apply 10 iterations of RAS and MLORAS (ours) as stationary algorithm and obtain solution for each method. Moreover, we run 10 iterations of FGMRES with MLORAS and RAS preconditioners on the initial guess and obtain predictions for both methods. The results are shown in Figure~\ref{fig:mloras-solution-plots} for MLORAS and  Figure~\ref{fig:ras-solution-plots} for RAS. 

\begin{figure}[!h]
  \centering
  \includegraphics[width = 1\textwidth]{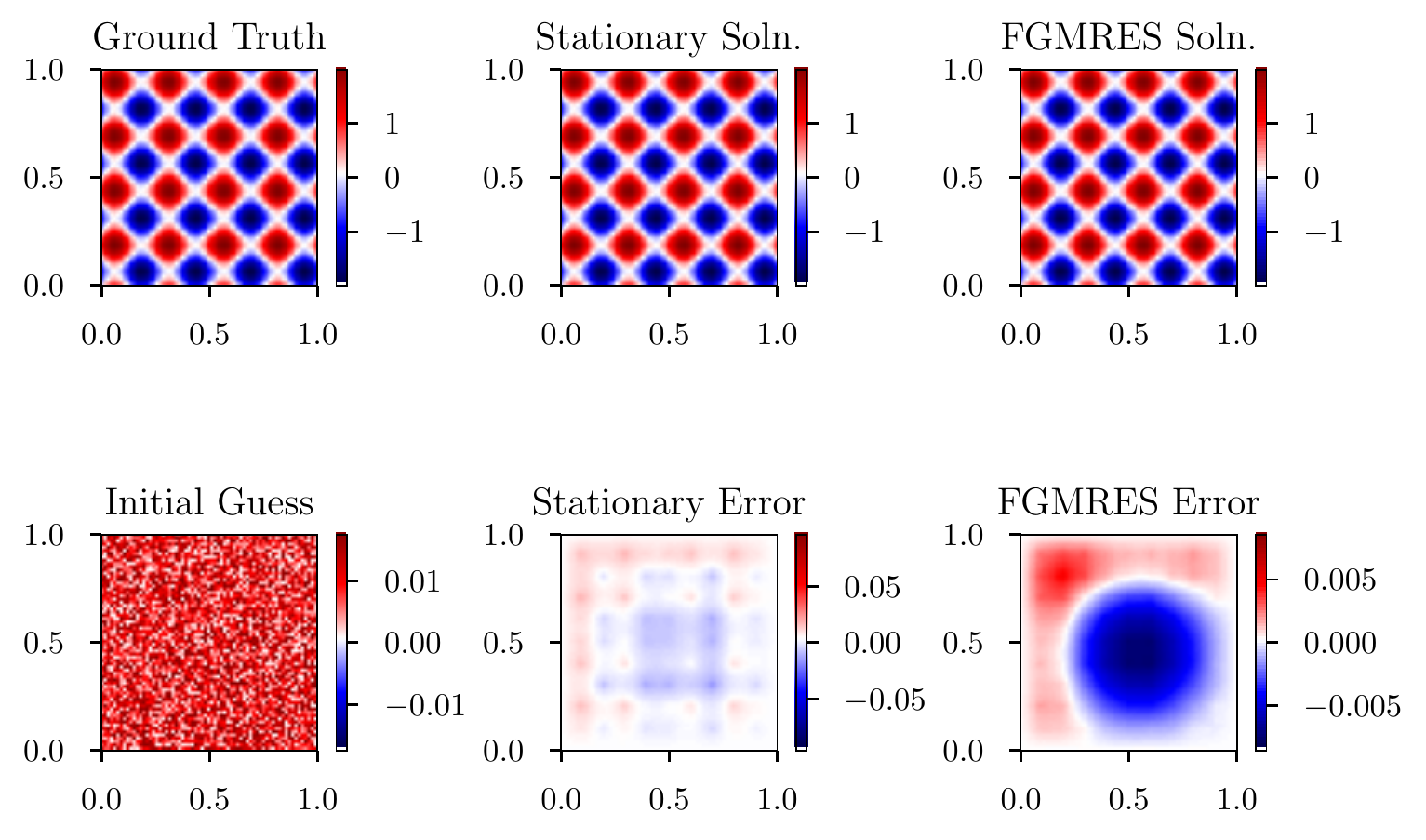}

    \caption{MLORAS (ours) solution plots. Top left: ground truth, top middle: MLORAS stationary solution after 10 iterations, top right: FGMRES solution with MLORAS preconditioner after 10 steps, bottom left: initial guess, bottom middle: error of MLORAS stationary solution ($L_2$ norm of error = 0.231), bottom right: error of FGMRES with MLORAS preconditioner solution ($L_2$ norm of error = 0.084).}
    
  \label{fig:mloras-solution-plots}
\end{figure}

\begin{figure}[!h]
  \centering
  \includegraphics[width = 1\textwidth]{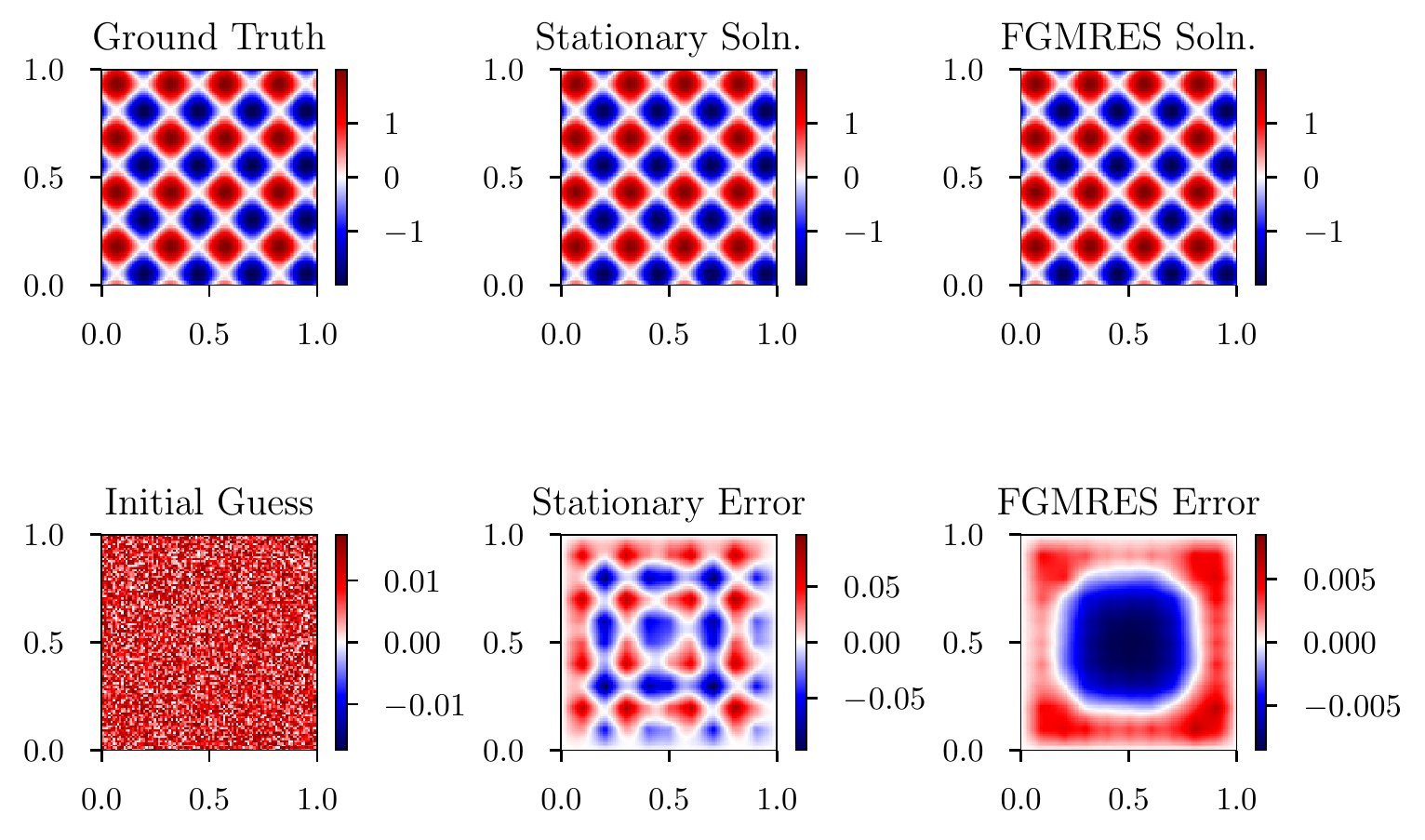}

 \caption{RAS solution plots. Top left: ground truth, top middle: RAS stationary solution after 10 iterations, top right: FGMRES solution with RAS preconditioner after 10 steps, bottom left: initial guess, bottom middle: error of RAS stationary solution ($L_2$ norm of error = 6.526), bottom right: error of FGMRES with RAS preconditioner solution ($L_2$ norm of error = 0.146).}  
 \label{fig:ras-solution-plots}
 
\end{figure}

\section{Neural network}

\subsection{Inputs and Output}

\paragraph{Inputs:} The GNN takes the grid $G$ as its input which has three main components, namely $D_{\text{node}}$ (node feature matrix), $D_{\text{edge}}$ (edge feature matrix), and $A$ (the graph adjacency matrix). Every node has a binary feature, and its value is one if it is on a boundary of a subdomain and zero otherwise. Therefore, for every node, the corresponding element in $D_{\text{node}}$ indicates whether that point is on a boundary or not. In other words, the binary node feature determines the grid decomposition. $D_{\text{edge}}$ consists of the edge values obtained from discretization of the underlying PDE.

\paragraph{Outputs:}After passing the input to the GNN architecture which consists of node and edge convolution blocks and is fully described in the following subsection, the learned edge weights are obtained for every edge in the grid. However, only the edges between the nodes on the subdomain boundaries (considering self-loops) are of our interest so we mask the rest of the edges. Therefore, the output of the GNN is the learned values for edges connecting nodes on the boundary of a subdomain which are referred to as interface values in the paper. For the $i$-th subdomain, the interface value matrix is referred to as $L_{i}^{\theta}$, where $\theta$ represents the GNN learnable parameters (see Equation~\ref{eq:GNNinout}). Figure~\ref{fig:structs_results10} shows an example of the sparsity pattern of an $L$ matrix for each of the identical subdomains of the $10\times10$ structured grid in Figure~\ref{fig:heatmap}. Also note that the interface values are the nonzero elements of the corresponding $L$ matrix.

\subsection{Architecture}
\label{appendix:GNN}

The overall architecture of the GNN is shown in Figure~\ref{fig:GNN_architecture}.
The GNN takes a graph as its input and sends node and edge features to the node convolution and edge feature preprocessing blocks, respectively, both of which are shown in Figure~\ref{fig:Node_conv_Edge_preprocess_block}.  Each node convolution block consists of a TAGConv layer with 2-size filters and 128 hidden units, followed by a ReLU activation, an instance norm layer, and a feature network block. The feature network block, as shown in Figure~\ref{fig:Feature_ResNet_Fig}, consists of 8 blocks with residual connections between each; furthermore, each of the blocks consists of a layer norm followed by a fully connected layer of size 128, followed by a ReLU activation, and finally another  fully connected layer of size 128. The edge feature preprocessing block takes the edge features and applies fully connected layers, ReLU nonlinearities, and instance normalization, before the graph convolution pass.
\begin{figure}[!h]
  \centering
  \includegraphics[width = 1\textwidth]{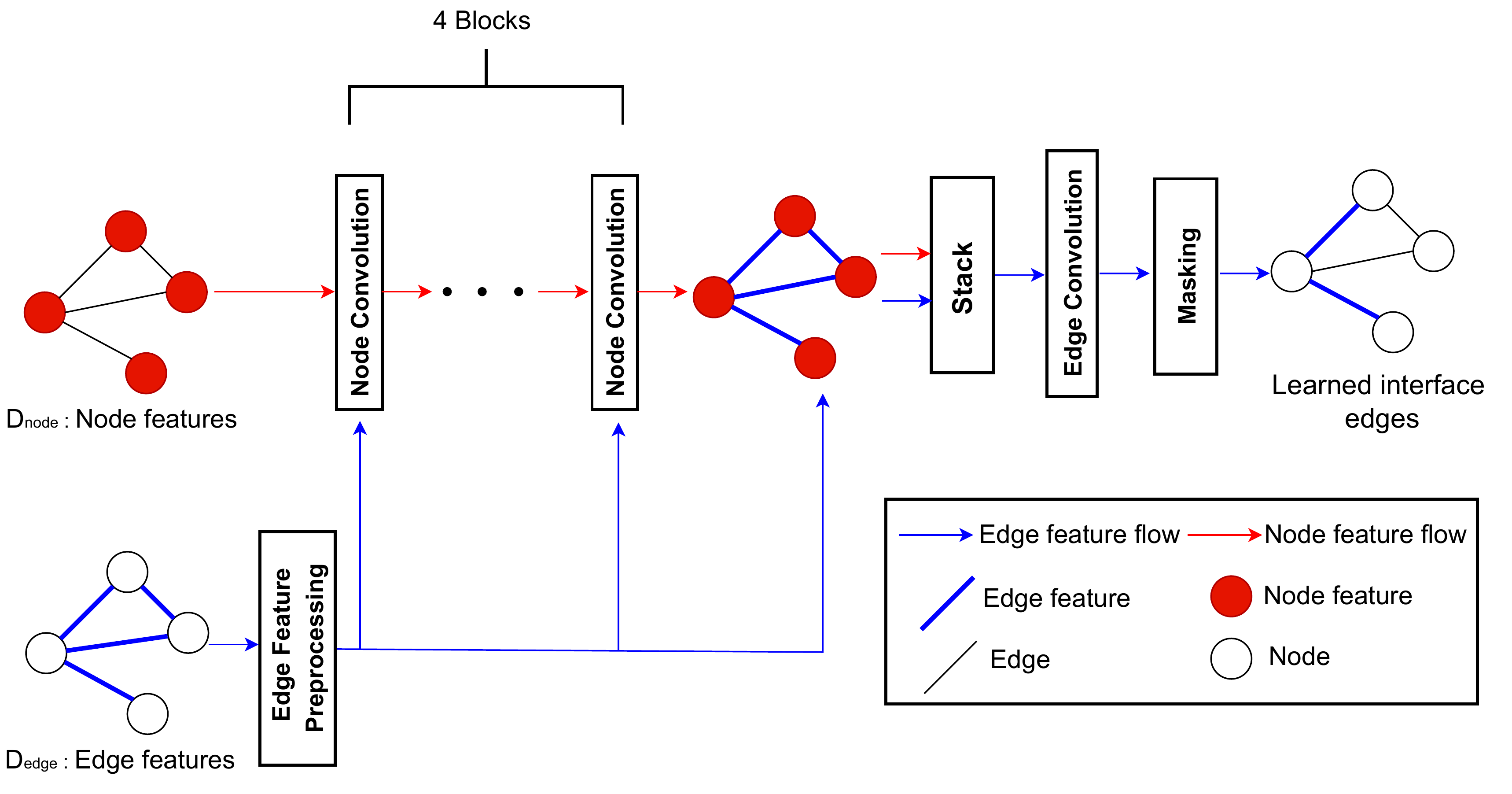}
    \caption{Overall GNN architecture.}
  \label{fig:GNN_architecture}
\end{figure}

\begin{figure}[!h]
  \centering

    \includegraphics[width = 0.405\textwidth]{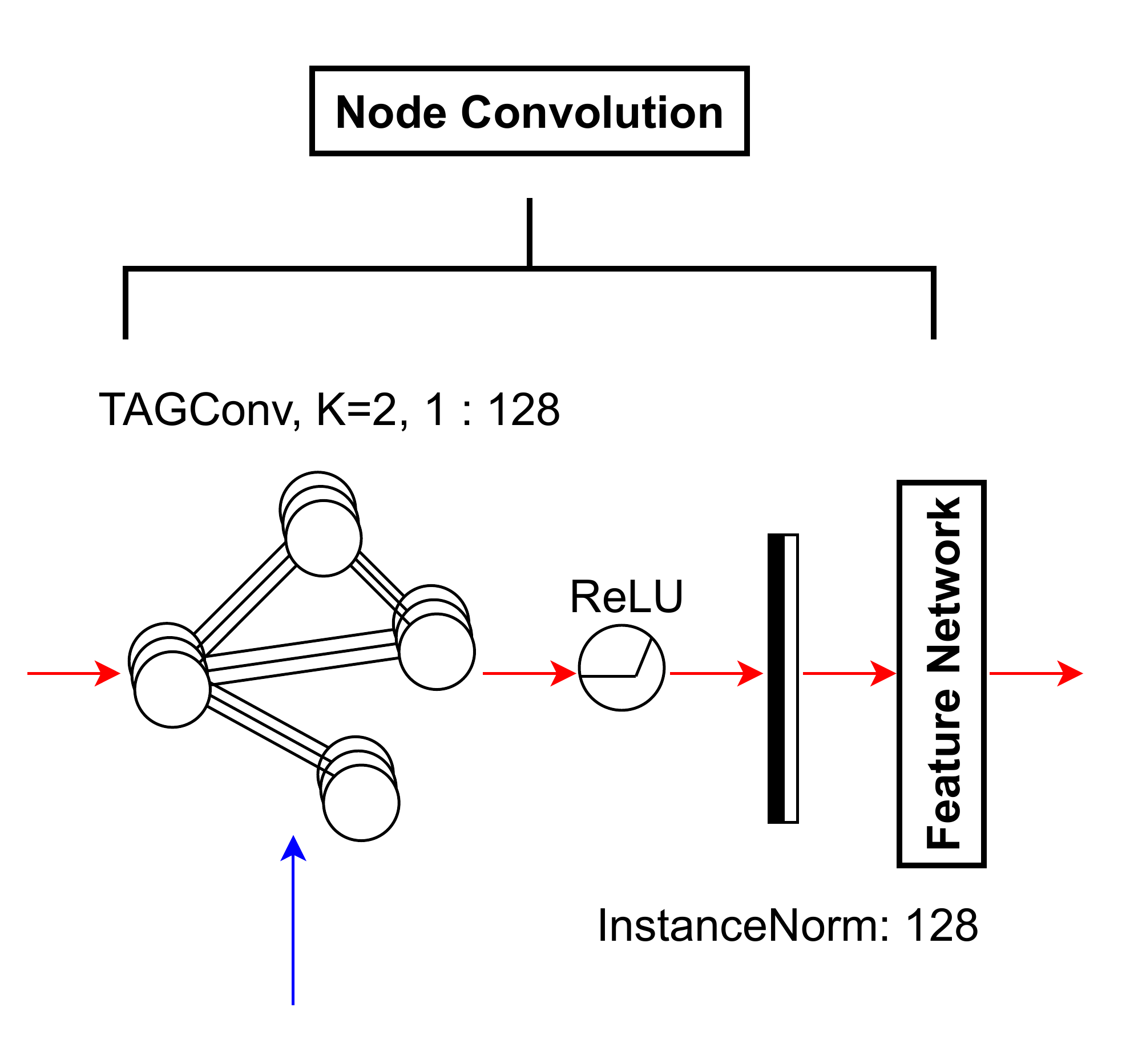}
    \includegraphics[width = 0.585\textwidth]{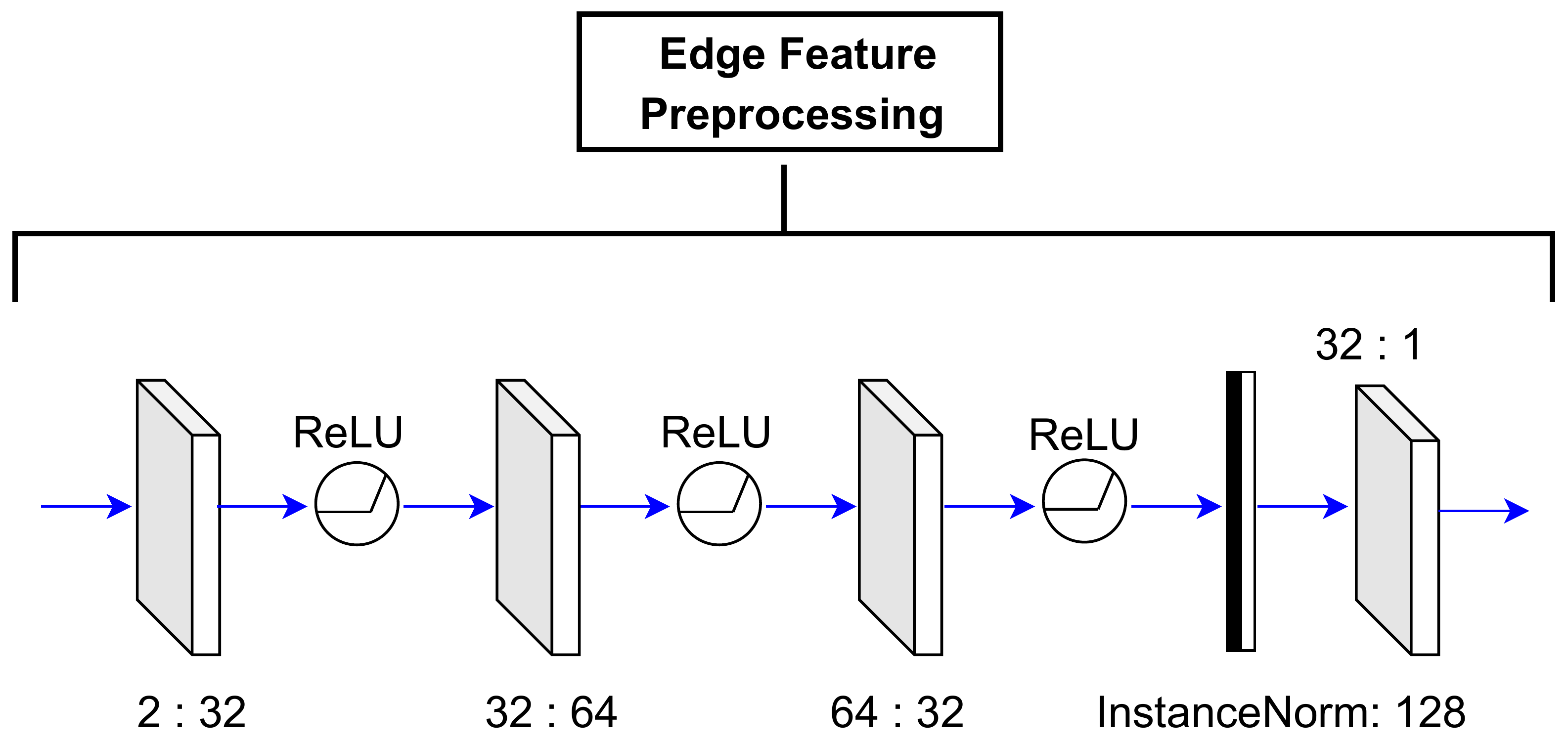}

    \caption{Left: Node convolution block. Right: Edge feature preprocessing block. }
  \label{fig:Node_conv_Edge_preprocess_block}
\end{figure}

\begin{figure}[!h]
  \centering

  \includegraphics[width = 0.8\textwidth]{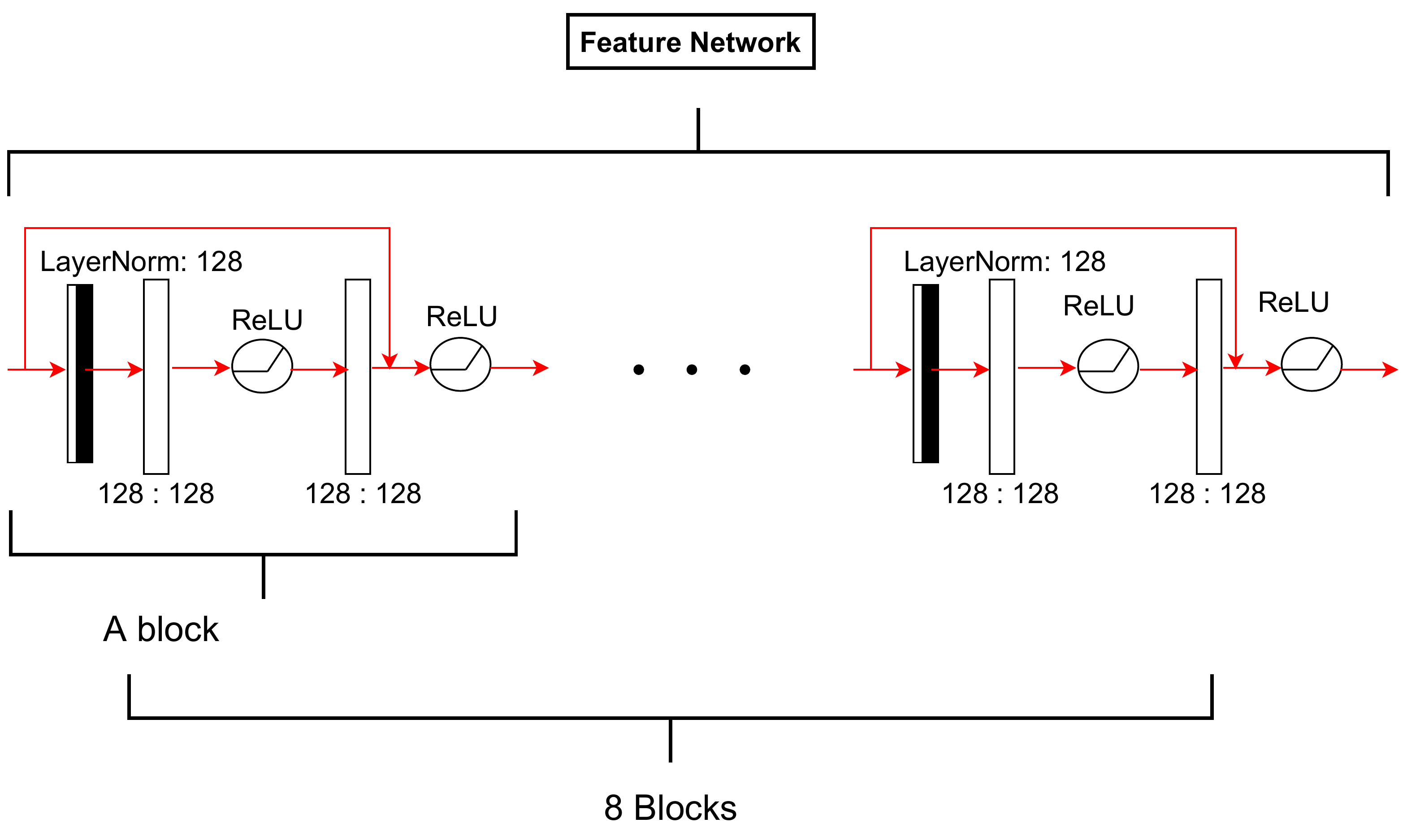}

    \caption{Feature network blocks.}
  \label{fig:Feature_ResNet_Fig}
\end{figure}

After passing through all of the node convolution blocks, the edge and node features are sent to a stack block.  This block simply stacks node features onto the edges adjacent to each node. For every edge, $(u, v)$, where $u$ and $v$ are the nodes on that edge, denote the node and edge features by $E_u, E_v, E_{(u,v)}$, which are the inputs to the stack block. The block then stacks these features and outputs them as the new edge features for the edge $(u,v)$.
Following the stack layer is the edge convolution block, depicted in Figure~\ref{fig:Edge_Conv_Fig}, which takes the stacked edge and node features and passes them through a series of fully connected layers, ReLU activation functions, and layer norms. It is noteworthy that the size of the input to the edge convolution block is 257, since following the description of the stack block, the two node features, each of size 128, and the edge feature of size 1 are stacked together.  The output from the edge convolution block is, finally, passed through a masking block that outputs the interface values.  This masks the interior edge values, i.e., takes the output of the GNN (one value per each edge in the graph) and multiplies those values that are not on the boundary of a subdomain by zero, to restrict the output from the GNN to the desired edges in the graph.
\begin{figure}[!h]
  \centering
    \includegraphics[width = 0.8\textwidth]{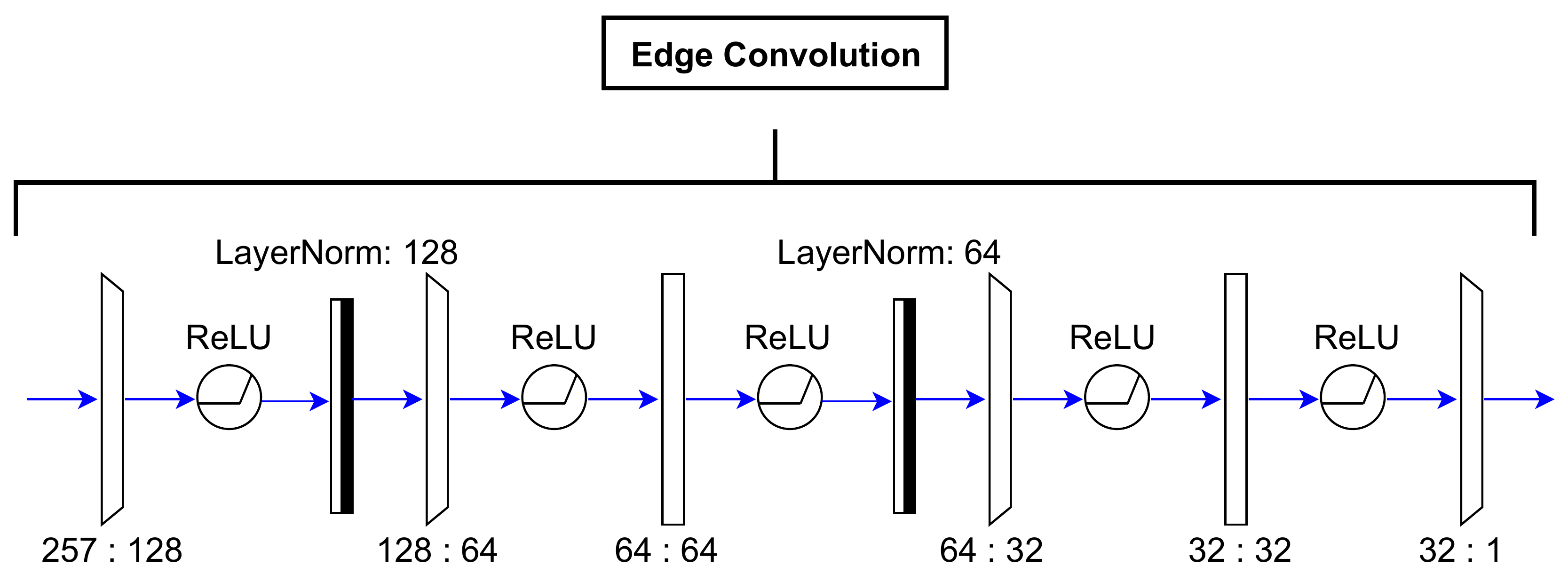}

    \caption{Edge convolution block.}
  \label{fig:Edge_Conv_Fig}
\end{figure}

\section{Poisson problem with discontinuous diffusion coefficient}
\label{sec:poisson}
We also consider the 2D Poisson problem with discontinuous diffusion coefficient which is formulated as follows:
\begin{equation}
- \div\kappa(x,y)\nabla u = f \quad\text{in} \;\Omega,\;\;\; \kappa(x,y) = \begin{cases}
    1000 & 0<x<0.5 \\
    1 & 0.5\le x<1.
  \end{cases}
\end{equation}where $\Omega$ is, as before, defined as a convex
  subset of $(0,1)\times(0,1)$ and
  $\kappa(x,y)$ is the discontinuous diffusion coefficient. For this
  problem, we consider nine domains with unstructured triangular
  grids, with sizes ranging from about 100 to over 30k nodes. The
  subdomains are generated using Lloyd aggregation (fully explained in
  Appendix~\ref{sec:Lloyd}), with a fixed ratio of 0.015. We note
  that the subdomains are not constrained in any way, and a single
  subdomain may contain parts of the domain with different diffusion
  coefficients. The results of our method compared with RAS baseline
  for both the stationary algorithm and the FGMRES preconditioner are
  shown in Figure~\ref{fig:poisson_coeff}, and show little
  qualitative difference with the earlier results for Helmholtz.

\begin{figure}[!h]
  \centering
  \includegraphics[width = 1\textwidth]{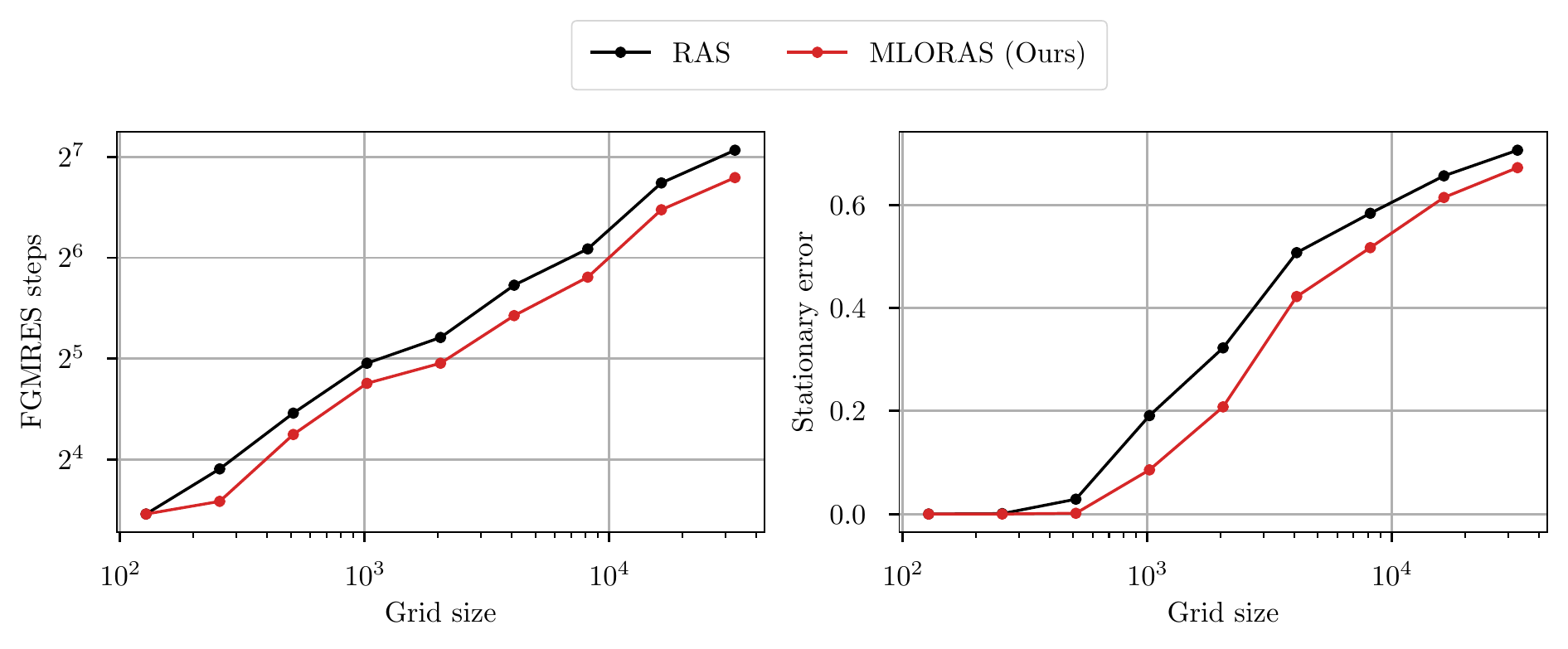}
    \caption{Discontinuous diffusion coefficient for Poisson problem on various size grids. Left: the number of preconditioned FGMRES steps required to solve the problem to within a relative error of $10^{-12}$. Right: error reduction by the stationary iteration after 10 iteration.}
  \label{fig:poisson_coeff}
\end{figure}

\section{Lloyd aggregation}
\label{sec:Lloyd}

Lloyd's algorithm~\cite{lloyd1982least} is a standard approach for partitioning data, closely related to $k$-means clustering, that can be used (for example) to find close approximations to centroidal Voroni tesselations.  Here, we use a modified form of Lloyd's algorithm, known as Lloyd aggregation~\cite{bell2008algebraic}, to partition a given set of degrees of freedom, $V$, into the non-overlapping subdomains, $V_i^0$, for $i\in\{1,2,\ldots,S\}$, needed as input to our algorithm.  (Note that, here, we change notation from the paper and use $V$ to denote the vertices in the graph associated with the matrix, rather than $D$ to denote the index set of degrees of freedom.)  Consider a 2D planar graph, $G$, the set of all edges $E$, the set of all of its nodes $V$, and $V_c\subseteq V$. The nodes in $V_c$ serve as the centers of ``clusters'' in the graph that will define the subdomains, $V_i^0$. These regions are obtained based on the closest center to each graph node, where the distance is measured by the number of edges covered in the shortest path between two nodes (denote distance in the graph between node $i$ and $j$ by $g_{ij}$). Define the centroid of a region as the farthest node from the boundary, breaking possible ties by random choice. We use a  modified version of the  Bellman-Ford algorithm,  commonly used for obtaining the nearest center to every node in $V$ and its associated distance~\cite{cormen2001introduction}. Let $\vec{n}$ be a list of graph nodes whose $i$-th element is the nearest center to the $i$-th node of the graph, and let $g_j$ be the graph distance from node $j$ to $n_j$; then, the modified Bellman-Ford algorithm is shown in Algorithm~\ref{alg:BellFord}.

\begin{algorithm}
\caption{Modified Bellman-Ford}\label{alg:BellFord}
\begin{algorithmic}[1]
\State \textbf{Input} $E$: The set all edges, $V$: The set of all nodes, $V_c$: The set of initial center nodes.
\State $g_i = \infty \;\;\forall_{i\in\{1,2,...,|V|\}}$
\State $n_i = -1 \;\forall_{i\in\{1,2,...,|V|\}}$
\For {$c\in V_c$}
\State $g_c \leftarrow 0$
\State $n_c \leftarrow c$
\EndFor
\While {True}
\State Finished $\leftarrow$ True
\For {$(i,j)\in E$}
\If {$g_i + g_{ij} < g_j$}
\State $g_j \leftarrow g_i + g_{ij}$
\State $n_j \leftarrow n_i$
\State Finished $\leftarrow$ False
\EndIf
\EndFor
\If {Finished}
\State \textbf{return} $\vec{g}, \vec{n}$
\EndIf
\EndWhile
\end{algorithmic}
\end{algorithm}

While this is an iterative computation, it has finite termination when the values in $\vec{g}$ and $\vec{n}$ stop changing.  After running this modified Bellman-Ford algorithm, Lloyd's algorithm modifies the clusters by selecting the centroid of every subdomain as its new center, then iterates, using the modified Bellman-Ford algorithm to calculate new distances and nearest centers. Given updated center positions, it forms the new subdomains. The full Lloyd algorithm is shown in Algorithm~\ref{alg:Lloyd}, where we define the set of border nodes, $B$, as the set of all nodes that are connected by an edge to a node that has a different nearest center node.  The key point here is that we use Modified Bellman-Ford to assign closest centers, then compute the set of border nodes, then find the new centers as those that are further from the border set within each of the original subdomains (using graph distances from $B$, but original center assignment in $\vec{n}$).

\begin{algorithm}
\caption{Lloyd Aggregation}\label{alg:Lloyd}
\begin{algorithmic}[1]
\State \textbf{Input} $K$: Number of iterations, $E$: The set of all edges, $V$: The set of all nodes, $V_c$: The set of initial center nodes.
\For {$i = 1,2,3,..., K$}
\State $\vec{g}, \vec{n} \leftarrow$ Modified Bellman-Ford$(E, V, V_{c})$
\State $B \leftarrow \emptyset$
\For {$(i,j)\in E$}
\If {$n_i \ne n_j$}
\State $B \leftarrow B \cup \{i,j\}$
\EndIf
\EndFor
\State $\vec{g},\vec{x} \leftarrow$ Modified Bellman-Ford$(E, V, B)$
\State $V_{c} \leftarrow \{i\in V: g_i> g_j \;\;\forall {n_i=n_j}\}$
\EndFor
\State \textbf{return} $\vec{n}$
\end{algorithmic}
\end{algorithm}
\paragraph{Time Complexity:}

Assuming each node's initial distance to a center node is bounded independently of $|V|$, and also assuming that each node's degree is bounded independently of $|V|$, Algorithm~\ref{alg:BellFord} runs a $|V|$-independent number of iterations to determine one nearest center node for every point. Thus, Algorithm~\ref{alg:BellFord} is $O(|V|)$ in our case. This is run a $|V|$-independent number of times in Algorithm~\ref{alg:Lloyd}, to generates the subdomains, resulting in an overall algorithmic cost of $O(|V|)$.

\end{appendices}

\end{document}